\documentclass[11pt]{article}

\usepackage[utf8]{inputenc}
\usepackage[T1]{fontenc}
\usepackage{amsmath,amssymb,amsthm}
\usepackage{mathtools}
\usepackage[shortlabels]{enumitem}
\usepackage{geometry}
\usepackage{graphicx}
\usepackage[hyperfootnotes=false]{hyperref}
\usepackage{natbib}
\usepackage{setspace}
\usepackage{xcolor}

\hypersetup{hidelinks}

\newcommand{\E}{\mathbb{E}}
\newcommand{\cS}{\mathcal{S}}
\newcommand{\1}{\mathbf{1}}

\newtheorem{theorem}{Theorem}
\newtheorem{prop}{Proposition}

\newtheorem{lm}[theorem]{Lemma}
\theoremstyle{definition}

\theoremstyle{remark}
\newtheorem{remark}[theorem]{Remark}

\title{Identification and Estimation of Nonsymmetric Determinantal Point Processes}
\author{Christian Gouri\'eroux\thanks{University of Toronto, Toulouse School of Economics and CREST; gouriero@ensae.fr} \, and
Yang Lu\thanks{Department of Mathematics and Statistics, Concordia University, Montreal, Canada; yang.lu@concordia.ca}}

\date{October 9, 2026}

\begin{document}

\maketitle

\begingroup
\renewcommand{\thefootnote}{}
\footnotetext{Gouri\'eroux gratefully acknowledges financial support of the ACPR Chair
``Regulation and Systemic Risks'', and the ERC DYSMOIA. Lu acknowledges support from
NSERC (grant no. RGPIN-2021-04144) and SSHRC (grant no. 100010986).}
\endgroup

\begin{abstract}
Determinantal point processes (DPPs) provide tractable distributions for finite random sets, but the usual symmetric-kernel specification imposes nonpositive pairwise dependence. Statistical research on nonsymmetric DPPs, which can accommodate positive dependence, remains limited. We study
identification and estimation for this class. First, we establish the
equivalence of two sufficient conditions for a nonsymmetric kernel to define
a valid DPP. Second, we develop the identification theory for
nonsymmetric DPPs.
Third, we develop cardinality-truncated maximum-likelihood estimation and
estimation by averaged stochastic gradient descent. {Under suitable identification and
regularity conditions, we establish consistency and asymptotic normality for
both estimators.} Simulations examine identification, the effects of truncation,
and the finite-sample performance of the estimators. We illustrate the methods with two cyber applications at markedly
different sampling frequencies: weekly sector-by-breach data and
second-by-second enterprise authentication data with more than five million set observations.
\end{abstract}

\section{Introduction}
\label{sec:intro}
The determinantal point process (DPP) is a tractable parametric distribution
for a binary vector $X=(X_1,\ldots,X_n)\in\{0,1\}^n$, or equivalently for the
finite random set $S=\{i:X_i=1\}\subseteq\{1,\ldots,n\}$. This family is parameterized by a matrix $L$, called the likelihood kernel. The literature has
traditionally focused on symmetric likelihood kernels $L$.
A consequence of symmetry is that the binary variables $X_i$, $i=1,\ldots,n$,
exhibit nonpositive pairwise dependence. This repulsive feature is less appropriate
for many economic and financial applications, such as joint defaults in credit
portfolios, co-occurrences of cyber incidents, and joint trading activities,
where binary events may instead display positive pairwise dependence
{\citep[see][for a random-set framework for such risks]{gourieroux2023risk}.
Recently, the DPP literature has started to investigate nonsymmetric DPPs
(NDPPs) \citep{brunel2018learning, gartrell2019learning,
arnaud2024determinantal}, but their statistical properties have received
little attention: which nonsymmetric kernels are valid, which are identified,
and how they can be estimated. This paper addresses these three questions.}

{Validity requires the likelihood kernel to be a $P_0$-matrix, that is,
to have nonnegative principal minors, a condition that is costly to check in
high dimension. We show that two tractable sufficient conditions used in the
literature coincide: the decomposition $L=\Sigma+A$, where $\Sigma$ is
symmetric positive semidefinite (SPSD) and $A$ is skew-symmetric, and the
requirement that the spectral norm of the associated parity kernel does not
exceed one. This class is not identified: diagonal similarity generates
continua of observationally equivalent kernels. With a diagonal symmetric
part, the kernel is identified up to sign changes in its off-diagonal entries.\par}

{For estimation, we condition the likelihood on the observed set
having at most $\bar n$ elements. This makes the fit of small configurations
insensitive to how larger ones are specified. For long data streams, we use
one-pass averaged stochastic gradient descent (ASGD). Truncated maximum
likelihood is consistent for the identified parameter set, and ASGD is
locally consistent. At a regular point, both estimators have the same normal
limit, and standard errors and tests based on the conditional information
must use the number of retained observations, not the calendar sample
size.\par}

Simulations support the truncated-likelihood inference results and document
the accuracy--computation trade-off of {ASGD}. In two cyber applications, NDPPs capture positive dependence and
outperform symmetric or independence benchmarks. {The authentication
database contains more than five million second-level set
observations; one ASGD pass over the 2.5 million used for the main fit takes
under seven minutes.}

Section~\ref{sec:models} introduces the NDPP model and compares two sufficient
conditions for kernel validity.
Section~\ref{sec:identification} addresses identification, and
Section~\ref{sec:inference} studies truncated maximum likelihood, ASGD, and inference for identified parameter
sets. Section~\ref{sec:simulation-study} reports the
Monte Carlo evidence, Section~\ref{sec:application} presents two cyber-event
applications, and Section~\ref{sec:conclusion} concludes. Longer proofs are gathered in the
appendix.

\section{Nonsymmetric DPP models}
\label{sec:models}

\subsection{Definitions and interpretations}
\label{sec:def}

The $L$-ensemble representation rests on the following identity
\citep[Theorem~2.1]{kulesza2012determinantal}.
\begin{lm}
\label{lem:principal-minor-sum}
			For any $n\times n$ matrix $L$ and subset
      $s\subseteq\{1,\ldots,n\}$, let $L_s$ denote the principal submatrix
	      indexed by $s$. Then the principal-minor identity is:
		\begin{equation}
			\label{allprincipalminors}
			\det(I_n+L)=\sum_{s\subseteq\{1,\ldots,n\}} \det L_s.%
		\end{equation}
		\end{lm}
If every principal minor of $L$ is nonnegative---that is, if $L$ is a
$P_0$-matrix {\citep[Chapter~4]{johnson2020matrix}}---then:
\begin{equation}
	\label{ldpp}
	p(s)= \frac{\det L_s}{\det(I_n+L)},
  \qquad s\subseteq\{1,\ldots,n\},
\end{equation}
defines a probability mass function, with the convention
$\det L_{\emptyset}=1$. The matrix $L$ is called the likelihood kernel.
Other kernels can also be introduced.

\noindent (i) For a $P_0$ likelihood kernel, the correlation kernel is defined
by:
\[
K=L(I_n+L)^{-1}
\qquad\Longleftrightarrow\qquad
L=(I_n-K)^{-1}-I_n,
\]
where $I_n+L$ is invertible because its determinant is positive by
Lemma~\ref{lem:principal-minor-sum}.  

The correlation kernel is in one-to-one correspondence with $L$ and satisfies:
\[
\mathbb{P}(s\subseteq S)=\det K_s,
\qquad s\subseteq\{1,\ldots,n\}.
\]
Its principal minors therefore give the inclusion probabilities. Moreover:
\begin{align}
\label{marginal}
\mathbb{P}(X_i=1)&=K_{ii}, \qquad i=1,\ldots,n, \\
\operatorname{Cov}(X_i,X_j)&=-K_{ij}K_{ji}, \qquad i\neq j.
\label{correlationkernel}
\end{align}
Thus symmetric $K$, or equivalently symmetric $L$, implies nonpositive pairwise
covariances.

\noindent (ii) \citet{arnaud2024determinantal} defines the parity kernel as:
\begin{equation}
\label{paritykernel}
M=I_n-2K \qquad \Longleftrightarrow \qquad K=\frac{1}{2}(I_n-M).
\end{equation}
The relation with the likelihood kernel is:
\begin{equation}
\label{cayley}
M=(I_n-L)(I_n+L)^{-1}
\qquad \Longleftrightarrow \qquad
L=(I_n-M)(I_n+M)^{-1},
\end{equation}
where $I_n+M=2(I_n+L)^{-1}$ is invertible.

Thus a DPP may be parameterized equivalently by its likelihood kernel $L$,
correlation kernel $K$, or parity kernel $M$. If $L$ is a $P_0$-matrix, then
$K$ is also a $P_0$-matrix, although this latter condition alone is not
sufficient for a nonsymmetric matrix $K$ to be a valid correlation kernel
\citep[Theorem~2.1]{arnaud2024determinantal}.

\subsection{Constrained DPP models}
\label{sec:constrained-models}

{A tractable sufficient condition for the $P_0$ property is due to
\citet{gartrell2019learning}, who decompose $L=\Sigma+A$ into its symmetric
part $\Sigma=(L+L')/2$, assumed positive semidefinite, and its skew-symmetric
part $A=(L-L')/2=-A'$. We denote the resulting class of kernels by:
\begin{equation}
\label{decomposition}
\mathcal L_\Sigma=\{L=\Sigma+A:\ \Sigma=\Sigma'\succeq0,\ A'=-A\}.
\end{equation}
}%
Related learning
and simulation algorithms are developed by
\citet{gartrell2019learning,gartrell2021scalable,anari2022sampling},
\citet{alimohammadi2021fractionally}, and \citet{han2022scalable}.

{\citet[Proposition~2.4]{arnaud2024determinantal} gives a second
sufficient condition, stated on the parity kernel: $\|M\|_2\leq1$ implies that
$L$ is a $P_0$-matrix. The next proposition shows that the two conditions
coincide.\par}

\begin{prop}
\leavevmode
\label{prop:admissible-parity}
\begin{enumerate}[label=(\roman*)]
\item If the symmetric part $\Sigma$ of $L$ is SPSD, then $L$ is a
$P_0$-matrix and defines a valid DPP.
\item A $P_0$-matrix $L$ has an SPSD symmetric part if and only if its parity
kernel satisfies $\|M\|_2\leq1$, where $\|M\|_2$ is the largest singular
value of $M$.
\end{enumerate}
\end{prop}
\begin{proof}
Part (i) follows from Proposition~4.5.2 of \citet{johnson2020matrix}.
{For part (ii), eq.~\eqref{cayley} gives $L(I+M)=I-M$, hence
$(I+M')L(I+M)=(I+M')(I-M)$; transposing,
$(I+M')L'(I+M)=(I-M')(I+M)$. Adding the two yields:
\begin{equation}
\label{important}
(I+M')(L+L')(I+M)=2(I-M'M).
\end{equation}
Since $I+M$ is invertible:}
\[
L+L'\succeq0
\quad\Longleftrightarrow\quad
I-M'M\succeq0
\quad\Longleftrightarrow\quad
\|M\|_2\leq1.
\]
\end{proof}

{The boundary $\|M\|_2=1$ contains the purely skew-symmetric
kernels: by \eqref{important}, $\Sigma=0$ if and only if $M$ is orthogonal.
These kernels, studied by \citet[Section~3.3]{arnaud2024determinantal}, have a
zero diagonal, so $\mathbb{P}(S=\{i\})=0$ for every $i$, which makes them
unsuitable for data.\par}

\section{Identification}
\label{sec:identification}
In the $L$-ensemble class studied here, a DPP distribution is parameterized
by a $P_0$ likelihood kernel $L$.
Without further restrictions, the map from $L$ to the distribution of $S$ is
not one-to-one; that is, $L$ is not identifiable. We first examine a necessary
order condition and then the rank conditions for identification.

\subsection{Order conditions}

{The probability simplex for $S$ has dimension $2^n-1$, whereas the
unrestricted kernel $L$ has $n^2$ free entries. The necessary order condition
$n^2\leq2^n-1$ fails for $2\leq n\leq4$, so the model is underidentified in
these dimensions. For $n\geq5$, the dimension count alone does not settle
identification.\par}

\subsection{Reduction to principal minors}

A fuller analysis factors the map from $L$ to the distribution of $S$ as:
\[
  L
  \longmapsto
  \{\det L_s : s\subseteq \{1,\ldots,n\}\}
  \longmapsto
\left \{  p_L(s)
  =
  \frac{\det L_s}{\det(I_n+L)} \right\}.
\]
The second map is injective because
$\det L_s=p_L(s)/p_L(\emptyset)$. Identification therefore reduces to
determining which kernels have the same principal minors. This is the principal
minor assignment problem considered next.

\subsection{The principal minor assignment problem}
\label{sec:id-general}

{For a general kernel, \citet{loewy1986principal} characterizes the
kernels sharing all principal minors:}
\begin{prop}[\citealp{loewy1986principal}, Theorem~1]
\label{prop:loewy}
Assume that $n\geq4$, that $L$ is irreducible,\footnote{A square matrix $L$ is reducible if there exists a nonempty proper subset $S\subset\{1,\ldots,n\}$ such that $L_{S^c,S}=0$, or equivalently if a simultaneous permutation of its rows and columns makes it block upper triangular. It is irreducible otherwise.} and that, for every bipartition $\{\alpha,\beta\}$ of $\{1,\ldots,n\}$ satisfying $|\alpha|\geq2$ and $|\beta|\geq2$:
\[
\operatorname{rank}(L_{\alpha,\beta})\geq2
\qquad\text{or}\qquad
\operatorname{rank}(L_{\beta,\alpha})\geq2.
\]
Then every matrix $L_1$ with the same principal minors as $L$ satisfies:
\begin{equation}
		\label{potentialtranspose}
L_1=D L D^{-1}, \text{ or } L_1=D L' D^{-1},
\end{equation}
for some invertible diagonal matrix $D$.
\end{prop}

Matrices related by the first equality in \eqref{potentialtranspose} are
diagonally similar (or D-similar); under the second equality, $L_1$ is diagonally
similar to $L'$.  %

{The class $\mathcal L_\Sigma$ guarantees validity
(Proposition~\ref{prop:admissible-parity}(i)) but, as we now show, not
identification.\par}

{Let $L=\Sigma+A\in\mathcal L_\Sigma$ with $\Sigma$ positive
definite, and let $D$ be an invertible diagonal matrix. The kernel $DLD^{-1}$
has the same principal minors as $L$, and its symmetric part,
$\tfrac12(DLD^{-1}+D^{-1}L'D)$, is continuous in $D$ and equals $\Sigma$ at
$D=I_n$. It therefore remains positive definite for all $D$ close enough to
$I_n$, so $\mathcal L_\Sigma$ contains a continuum of observationally
equivalent kernels.
{For example,
$L_\star=\left(\begin{smallmatrix}1&5\\-1&1\end{smallmatrix}\right)$ is a
$P_0$-matrix whose symmetric part has eigenvalues $3$ and $-1$, so
$L_\star\notin\mathcal L_\Sigma$. With $D_1=\operatorname{diag}(1,\sqrt5)$
and $D_2=\operatorname{diag}(1,2)$, its transforms $L_k=D_kL_\star D_k^{-1}$
are:
\[
L_1=\begin{pmatrix}1&\sqrt5\\-\sqrt5&1\end{pmatrix},
\qquad
L_2=\begin{pmatrix}1&5/2\\-2&1\end{pmatrix}.
\]
Their symmetric parts are $I_2$ and
$\left(\begin{smallmatrix}1&1/4\\1/4&1\end{smallmatrix}\right)$, so both
belong to $\mathcal L_\Sigma$, and they have the same principal minors as
$L_\star$. The equivalence class of $L_1$ thus contains a continuum of kernels
in $\mathcal L_\Sigma$ and kernels, such as $L_\star$, outside it. We next
consider a subclass of $\mathcal L_\Sigma$ that reduces this identification
issue.}\par}

\subsection{The case \texorpdfstring{$L=\Lambda+A$}{L = Lambda + A} with diagonal \texorpdfstring{$\Lambda$}{Lambda}}
\label{sec:diag-skew}
Let us now restrict the symmetric part of $L$ to be diagonal, writing
$\Sigma=\Lambda=\operatorname{diag}(\lambda_1,\ldots,\lambda_n)$, where
$\lambda_i\geq0$. Equivalently:
\begin{equation}
\label{oppositesign}
L_{ij}+L_{ji}=0, \qquad i\neq j.
\end{equation}

{We denote this subclass of $\mathcal L_\Sigma$ by
$\mathcal L_\Lambda$ and call it the diagonal-skew model.} This restriction removes $n(n-1)/2$ parameters and substantially simplifies
the identification.

{Restriction~\eqref{oppositesign} is the magnitude symmetry
studied by \citet{brunel2018learning}, $L_{ij}=\epsilon_{ij}L_{ji}$ for
$i\neq j$ with $\epsilon_{ij}\in\{-1,1\}$, imposed on the likelihood kernel
with every $\epsilon_{ij}=-1$. \citet{brunel2018learning} imposes it on the
correlation kernel of a signed DPP, where it eases recovery of the kernel from
its principal minors but does not guarantee that the kernel defines a DPP, so
validity must be checked separately. Because
$\mathcal L_\Lambda\subset\mathcal L_\Sigma$, the diagonal-skew model
guarantees validity (Proposition~\ref{prop:admissible-parity}(i)) and leaves
only finitely many observationally equivalent kernels.\par}

{Within $\mathcal L_\Lambda$, the off-diagonal blocks of $L$ are
those of its skew part $A$, so the two cross-block ranks in
Proposition~\ref{prop:loewy} coincide. Only one cross-block rank condition
is therefore needed. We apply to the skew part $A$ a result of \citet[Theorem~7]{boussairi2015skew}, which sharpens
Proposition~\ref{prop:loewy} for skew-symmetric matrices with nonzero
off-diagonal entries.\par}

\begin{prop}
\label{lda}
Let $n\geq4$, and write $L_0=\Lambda_0+A_0$, where $\Lambda_0$ is diagonal
with nonnegative entries and $A_0'=-A_0$.
Assume that $L_0$ is off-diagonal full, meaning that all its off-diagonal entries
are nonzero, and that, for every bipartition
$\{\alpha,\beta\}$ of $\{1,\ldots,n\}$ with
$|\alpha|\geq2$ and $|\beta|\geq2$, the following condition holds:
\begin{equation}
\label{simplified}
\operatorname{rank}\bigl((L_0)_{\alpha,\beta}\bigr)\geq2.
\end{equation}
Within {$\mathcal L_\Lambda$}, the
kernels {whose principal minors of order at most four coincide
with those of $L_0$ (and hence all of whose principal minors do)} are exactly the
$2^{n-1}$ distinct matrices $D L_0D^{-1}$ and the $2^{n-1}$ distinct
matrices $D L_0'D^{-1}$, where $D$ ranges over the sign matrices:
\[
D=\operatorname{diag}(1,d_2,\ldots,d_n),
\qquad d_j\in\{-1,1\}.
\]
The two sets of matrices are disjoint, so the equivalence class contains exactly
$2^n$ kernels.

Within either set, impose:
\[
L_{1j}>0,\qquad j=2,\ldots,n.
\]
These restrictions select a unique diagonal-sign representative. The two row-normalized
representatives associated with $L_0$ and $L_0'$ have opposite signs in every
off-diagonal entry outside the first row and column. Hence the additional
restriction $L_{23}>0$ selects one of the two remaining kernels;
off-diagonal fullness ensures that $L_{23}\ne0$.
\end{prop}
\begin{proof}
{Equal singleton minors give equal diagonals, after which the skew
parts have the same principal minors of order at most four.
\citet[Theorem~7]{boussairi2015skew} then gives the sign-similarity or
transposition relation.} Appendix~\ref{app:proofs-est}
gives the complete proof, including the counting and normalization arguments.
\end{proof}

When every diagonal entry of the diagonal matrix $\Lambda$ is positive, this
specification can equivalently be written as:
\begin{equation}
	\label{qualitydiversity}
	L=\Lambda^{1/2}(I_n+R)\Lambda^{1/2},
\end{equation}
where $R=\Lambda^{-1/2}A\Lambda^{-1/2}$ is skew-symmetric. The factorization
is analogous to the quality--diversity decomposition used for symmetric DPPs
\citep{kulesza2012determinantal,anzai2026determinantal}: matrix $\Lambda$ controls
item qualities, while $R$ controls the dependencies. For $i\neq j$, let us consider the odds ratio conditional
on all remaining components being zero,
$p(\{i,j\})p(\varnothing)/[p(\{i\})p(\{j\})]$. Under
\eqref{qualitydiversity}, we obtain:
\[
p(\varnothing)=\frac{1}{\det(I+L)},
\qquad
p(\{i\})=\frac{\lambda_i}{\det(I+L)},
\qquad
p(\{j\})=\frac{\lambda_j}{\det(I+L)},
\]
so that $\lambda_i=p(\{i\})/p(\varnothing)$ is the quality of item $i$.
Moreover:
\begin{align*}
	p(\{i,j\})
	&=
	\frac{\det L_{\{i,j\}}}{\det(I+L)}
	=
	\frac{\lambda_i\lambda_j}{\det(I+L)}
	\det
	\begin{pmatrix}
		1&r_{ij}\\
		-r_{ij}&1
	\end{pmatrix}
	=
		\frac{\lambda_i\lambda_j(1+r_{ij}^{2})}{\det(I+L)}.
\end{align*}
Consequently:
\[
\frac{p(\{i,j\})p(\varnothing)}
{p(\{i\})p(\{j\})}
=
1+r_{ij}^{2} \geq 1.
\]
{This conditional odds ratio does not depend on the qualities. Hence
the singleton and pair configurations identify every $\lambda_i$ and every
$|r_{ij}|$, but not the signs of the $r_{ij}$. The marginal inclusion
probabilities $\operatorname{diag}K$ depend on both $\Lambda$ and $R$, but
the next result shows that, for any fixed $R$, $\Lambda$ can still be chosen
to match any vector of marginal inclusion probabilities.}
 
 \begin{prop}
 \label{prop:marginal-parameterization}
 For every fixed skew-symmetric matrix $R$, the map from the positive
 diagonal matrix $\Lambda$ to the marginal-inclusion vector
 $\operatorname{diag}K$ is a bijection onto $(0,1)^n$. Hence {the
 kernels in $\mathcal L_\Lambda$ with positive diagonal} may equivalently be parameterized by
 $(\Lambda,A)$ or by $(\pi,R)$, where the parameterizations are related by:
 \[
 \pi=\operatorname{diag}K,
 \qquad
 \Lambda=\Lambda(\pi,R).
 \]\end{prop}
 \begin{proof}
 See Appendix~\ref{app:proof-marginal-parameterization}.
 \end{proof}
 Thus fixing $R$ leaves the marginal inclusion probabilities unrestricted.

\section{Statistical inference}
\label{sec:inference}

Let $S_1,\ldots,S_T$ be independent and identically distributed (i.i.d.) random sets
from a DPP parameterized by $\theta$, with probability mass function:
\begin{equation}
\label{eq:parametric-dpp-mass}
  p(s;\theta)
  =
  \frac{\det L_s(\theta)}{\det(I_n+L(\theta))}.
\end{equation}
The parameter $\theta$ may index the kernels $L$, $K$, or $M$ and may {consist
of the free entries of the kernel itself, as for $\mathcal L_\Sigma$ and
$\mathcal L_\Lambda$ in Section~\ref{sec:identification}.}

\subsection{Likelihood-based estimation criteria}
We consider truncated maximum likelihood and its stochastic-gradient descent.
\subsubsection{Conditional truncated log-likelihood}
\label{sec:conditional-truncated-likelihood}
For an integer $\bar n\leq n$, let $\cS_{\bar n}$ be the set of all subsets of
$\{1,\ldots,n\}$ with cardinality $|s|$ at most $\bar n$:
\[
  \cS_{\bar n}
  =
  \{s\subseteq \{1,\ldots,n\}: |s|\leq \bar n\}.
\]
For each subset $s$, define:
\[
  Z_t(s)=\1\{S_t=s\},\qquad
  N_T(s)=\sum_{t=1}^T Z_t(s),\qquad
  \hat p_T(s)=\frac{N_T(s)}{T}.
\]
Write $\pi_{\bar n}(\theta)=\Pr_\theta\{S\in\cS_{\bar n}\}$ for the
retention probability and
$N_T=\sum_{t=1}^T\1\{S_t\in\cS_{\bar n}\}$ for the number of retained
observations. For $s\in\cS_{\bar n}$, define the conditional probability:
\[
q(s;\theta)
=\frac{p(s;\theta)}{\pi_{\bar n}(\theta)}
=\frac{\det L_s(\theta)}{\sum_{r\in\cS_{\bar n}}\det L_r(\theta)}.
\]
The conditional truncated log-likelihood is:
\begin{align}
  \ell_T(\bar n;\theta)
  &=
  \sum_{t=1}^T
  \sum_{s: |s| \leq \bar n}
  Z_t(s)
  \log q(s;\theta),
 \\
  &=
  T
  \sum_{s: |s| \leq \bar n}
  \hat p_T(s)
  \log\left[
    \frac{\det L_s(\theta)}{\sum_{s^*: |s^*| \leq \bar n}\det L_{s^*}(\theta)}
  \right].
\end{align}
The truncated normalizer need not be evaluated by enumerating all retained
subsets. Indeed,
\[
\det(I_n+zL(\theta))=\sum_{k=0}^{n}c_k(\theta)z^k.
\]
Then the denominator above is
$\sum_{k=0}^{\bar n}c_k(\theta)$, and Newton's identities give the recursion
\citep{kalman2000matrix}:
\begin{equation}
\label{eq:newton-truncated-normalizer}
c_0(\theta)=1,
c_k(\theta)
=
\frac{1}{k}\sum_{j=1}^{k}(-1)^{j-1}
c_{k-j}(\theta)\operatorname{tr}\!\left(L(\theta)^j\right),
\qquad k=1,\ldots,\bar n,
\end{equation}
where $\operatorname{tr}$ denotes the trace of a matrix.
When $\bar n=n$, the denominator is $\det(I_n+L(\theta))$, recovering the
standard log-likelihood function.

Alternatively, one may maximize the censored likelihood criterion:
\[
\ell_T^{\mathrm{cens}}(\theta)
=
\sum_{t:S_t\in\cS_{\bar n}}\log p(S_t;\theta)
+
\sum_{t:S_t\notin\cS_{\bar n}}
\log\!\left[1-\pi_{\bar n}(\theta)\right].
\]
The Newton recursion makes $\pi_{\bar n}(\theta)$ readily available.
Conditioning also has a robustness interpretation.
{Under the data-generating process, let $q_0$ denote the
conditional mass on $\cS_{\bar n}$ and $\pi_0$ the retention probability.}
The population conditional
criterion is:
\[
\pi_0\sum_{s\in\cS_{\bar n}}q_0(s)\log q(s;\theta).
\]
It depends on the retained model,
but not on the distribution beyond the cutoff. Under correct conditional
specification, it is maximized wherever $q(\cdot;\theta)=q_0$; its maximizing set is
unaffected by the true retention probability $\pi_0$, which enters only as a
positive multiplicative factor. The censored likelihood criterion also fits
$\pi_{\bar n}(\theta)$ to $\pi_0$. This additional information can improve
efficiency under correct full specification, but misspecification of the
upper tail can shift the population maximizer of the censored criterion
away from the best fit to the retained configurations. Thus $\bar n$ controls
an efficiency--tail-robustness trade-off.

\subsection{Algorithmic approach}
The maximum-likelihood estimates are computed numerically using BFGS-type
quasi-Newton algorithms, with analytic gradients and multiple starting values.

{\emergencystretch=1em
We also use one-pass ASGD to approximate the truncated maximum-likelihood estimator
(MLE). Each update follows the gradient of an observation-level log-likelihood
contribution, and the reported estimate averages the resulting iterates
\citep{polyak1992acceleration}.
Define the component criterion:\par}
\begin{equation}
  \ell(S_t,\bar n;\theta)
  =
  \1\{|S_t|\leq \bar n\}
  \log\left[
    \frac{\det L_{S_t}(\theta)}{\sum_{s^*: |s^*| \leq \bar n}\det L_{s^*}(\theta)}
  \right].
\end{equation}
The ASGD recursion is:
\begin{align}
  \hat\theta_t
  &=\hat\theta_{t-1}+\gamma_t
    \nabla_\theta\ell(S_t,\bar n;\hat\theta_{t-1}), \\
  \bar\theta_t
  &=\frac{t-1}{t}\bar\theta_{t-1}+\frac{1}{t}\hat\theta_t,
  \qquad t=1,\ldots,T,
\end{align}
with initial value $\hat\theta_0$ and learning rate $\gamma_t=\gamma t^{-c}$, where
{$\gamma>0$ and $c\in(1/2,1]$. This range ensures
$\sum_t\gamma_t=\infty$ and $\sum_t\gamma_t^2<\infty$, conditions used for
consistency.} Because the component criterion is zero outside
$\cS_{\bar n}$, $\hat\theta_t$ changes only when $|S_t|\leq\bar n$.
The averaged output $\bar\theta_T$ depends on
$(\bar n,\hat\theta_0,\gamma,c)$. Polyak--Ruppert
averaging dampens iteration noise and, under the conditions stated below,
recovers the first-order efficiency of maximum likelihood. Other schemes
include tail averaging, exponential moving averages, and stochastic weight
averaging \citep[see][for a review]{lakshmanan2026short}.

\subsection{Consistency and representative selection}

Assume the DPP model is correctly specified, with true parameter $\theta_0$.
For truncation
level $\bar n$, define:
\begin{equation}
  \Theta_0(\bar n)
  =
  \left\{
  \theta:
  q(s;\theta)=q(s;\theta_0)
  \text{ for every }s\in\cS_{\bar n}
  \right\}.
\end{equation}
This set contains the true value $\theta_0$. It is a singleton if the retained conditional
probabilities identify $\theta_0$ uniquely.

\begin{remark}[Principal-minor characterization of the truncated identified set]
\label{rem:truncated-id}
Because $\varnothing\in\cS_{\bar n}$ and
$\det L_{\varnothing}=1$, the equality of the truncated conditional masses at the
empty set implies the equality of the truncated normalizers:
\[
q(\varnothing;\theta)
=\left\{\sum_{r\in\cS_{\bar n}}\det L_r(\theta)\right\}^{-1}
=q(\varnothing;\theta_0).
\]
The equality of the other conditional masses then yields:
\[
\Theta_0(\bar n)
=\left\{\theta:\det L_s(\theta)=\det L_s(\theta_0)
\text{ for every }s\in\cS_{\bar n}\right\}.
\]
Thus truncation identifies the retained principal minors themselves.
{In $\mathcal L_\Lambda$, parameterized by
$\theta=(\lambda_i,a_{ij}:i<j)$, any $\bar n\geq2$ retains the singleton and
pair configurations used in Section~\ref{sec:diag-skew}. Hence
$\Theta_0(\bar n)$ is finite, with at most $2^{n(n-1)/2}$ sign patterns of
the $a_{ij}$. Triples add no sign information: odd-order skew-symmetric
minors vanish, so
$\det L_{\{i,j,k\}}=\lambda_i\lambda_j\lambda_k+\lambda_ia_{jk}^2
+\lambda_ja_{ik}^2+\lambda_ka_{ij}^2$, and $\Theta_0(2)=\Theta_0(3)$. Under
the conditions of Proposition~\ref{lda}, which uses only minors of order at
most four, $\Theta_0(\bar n)=\Theta_0(n)$ for every $4\leq\bar n\leq n$: both
consist of the parameters in $\Theta$ of the $2^n$ kernels of that
proposition.}
\end{remark}

{The truncated MLE may be set-valued:}
\[
  \hat\Theta_T(\bar n)
  =
  \operatorname*{arg\,max}_{\theta\in\Theta}
  \ell_T(\bar n;\theta).
\]
{For a fixed data sequence and fixed
$(\bar n,\hat\theta_0,\gamma,c)$, the ASGD recursion produces a single
sequence of parameter estimates.}

\begin{prop}[Consistency]
\phantomsection
\label{prop:consistency}
{
Fix $\bar n$ and define the population criterion:
\[
  \ell_\infty(\bar n;\theta)
  =\E_{\theta_0}[\ell(S_t,\bar n;\theta)].
\]
\begin{enumerate}[label=(\roman*)]
\item \textit{(Truncated MLE)} Suppose $\Theta$ is compact and
$\ell_\infty(\bar n;\theta)$ is finite and continuous on $\Theta$. Then
\[
  \sup_{\tilde\theta\in\hat\Theta_T(\bar n)}
  \inf_{\theta^\star\in\Theta_0(\bar n)}
  \|\tilde\theta-\theta^\star\|\longrightarrow0
  \qquad\text{almost surely as }T\to\infty.
\]
\item \textit{(ASGD)} Fix
$\theta_0^\star\in\Theta_0(\bar n)\cap\operatorname{int}(\Theta)$ and define:
\[
  \mathcal H
  =-\partial_{\theta\theta'}^2
    \ell_\infty(\bar n;\theta_0^\star).
\]
Suppose $q(s;\theta)$ is positive and twice continuously differentiable near
$\theta_0^\star$ for every $s\in\cS_{\bar n}$, and $\mathcal H$ is positive
definite. For $\gamma_t=\gamma t^{-c}$, with $\gamma>0$ and
$c\in(1/2,1]$, there exists $r>0$ such that
$B_r=\{\theta:\|\theta-\theta_0^\star\|\leq r\}
\subset\operatorname{int}(\Theta)$ and the following holds: if the raw
iterates are almost surely eventually contained in $B_r$, or if every update
is projected onto $B_r$, then
$\hat\theta_T\to\theta_0^\star$ and $\bar\theta_T\to\theta_0^\star$ almost
surely as $T\to\infty$.
\end{enumerate}
}
\end{prop}
\begin{proof}
See Appendix~\ref{app:proof-consistency}.
\end{proof}

{
{Part~(i) is outer consistency: every maximizer approaches
$\Theta_0(\bar n)$, but need not converge to a particular point of it. When
$\Theta_0(\bar n)$ is a singleton, every measurable selection from
$\hat\Theta_T(\bar n)$ is consistent.\par}

{Part~(ii) is local, and a global version would not deliver
consistency. The usual global guarantee for stochastic gradient methods is
convergence to a stationary point of the population criterion, and here such
points can lie outside $\Theta_0(\bar n)$. In $\mathcal L_\Lambda$,
$\det(\Lambda_s+A_s)=\det(\Lambda_s-A_s)$ by transposition, so every
skew-score coordinate vanishes at $A=0$: a recursion started at independence
($A=0$) never leaves it. If $\bar n\geq2$ and $A_0\neq0$, the retained
singleton and pair minors imply that the manifold $\{A=0\}$ is disjoint from
$\Theta_0(\bar n)$. Along $A=0$, writing $\alpha_i=\log\lambda_i$ yields the
regular conditional exponential family:
\[
  q_\alpha^{\mathrm{diag}}(s)
  \propto\exp\!\left\{\sum_i\alpha_i\1\{i\in s\}\right\},
  \qquad s\in\cS_{\bar n}.
\]
If the population criterion restricted to this family has a maximizer
$(\alpha^\star,0)\in\operatorname{int}(\Theta)$, its diagonal score is zero;
together with the vanishing skew score, this makes it a stationary point of
the population criterion outside $\Theta_0(\bar n)$.\par}

{ASGD is locally consistent within the attraction basin of an
isolated regular representative, so with several regular representatives the
starting value selects one. Along a continuum of observationally equivalent
parameters, each update is orthogonal to the set of parameters equivalent to
the current iterate. In Section~\ref{sec:simulation-identification}, where
this set is a coordinate direction, the unidentified coordinates therefore
keep their starting values.\par}

}

\subsection{Asymptotic distribution}

Asymptotic normality requires the estimator to converge pointwise to a selected
representative $\theta_0^\star$ of $\Theta_0(\bar n)$ in the interior of the
parameter set. {The following result applies when this
representative is first-order regular, in the sense that the Jacobian of the
conditional probability map has full column rank at $\theta_0^\star$.}

The conditional Fisher information is:
\[
  \mathcal I_{\bar n}(\theta)
  =\E_\theta\!\left[\nabla_\theta\log q(S;\theta)\,
    \nabla_\theta\log q(S;\theta)'\mid S\in\cS_{\bar n}\right]
  =\sum_{s\in\cS_{\bar n}}\frac{1}{q(s;\theta)}
    \nabla_\theta q(s;\theta)
    \nabla_\theta q(s;\theta)' .
\]

\begin{prop}[Regular pointwise asymptotic normality]
\label{prop:asnorm}
{
Assume the DPP model is correctly specified and fix
$\theta_0^\star\in\Theta_0(\bar n)\cap\operatorname{int}(\Theta)$.
Write
$\mathcal H=-\partial_{\theta\theta'}^2
\ell_\infty(\bar n;\theta_0^\star)$. Suppose
$q(s;\theta)$ is positive and three times continuously
differentiable in a neighborhood of $\theta_0^\star$ for every
$s\in\cS_{\bar n}$, and that
$\mathcal I_{\bar n}(\theta_0^\star)$ is nonsingular. Let $\tilde\theta_T$
be one of the following:
\begin{enumerate}[label=(\alph*)]
\item a measurable selection from $\hat\Theta_T(\bar n)$ satisfying
$\tilde\theta_T\xrightarrow{p}\theta_0^\star$; or
\item the ASGD average
$\bar\theta_T(\hat\theta_0,\bar n,\gamma,c)$, whose raw iterates satisfy
$\hat\theta_T\to\theta_0^\star$ almost surely, with
$\gamma_t=\gamma t^{-c}$, $\gamma>0$, and $c\in(1/2,1)$. If projection is
used, suppose $\theta_0^\star$ lies in the interior of the projection set.
\end{enumerate}
Then
\begingroup
\setlength{\abovedisplayskip}{4pt}
\setlength{\abovedisplayshortskip}{4pt}
\setlength{\belowdisplayskip}{4pt}
\setlength{\belowdisplayshortskip}{4pt}
\begin{align}
  \sqrt{T}\,\bigl(\tilde\theta_T-\theta_0^\star\bigr)
  &\xrightarrow{d}
  N\!\left(0,\;
    [\pi_{\bar n}(\theta_0)\mathcal I_{\bar n}(\theta_0^\star)]^{-1}\right),
    \label{eq:prop6-collapsed}\\
  \sqrt{N_T}\,\bigl(\tilde\theta_T-\theta_0^\star\bigr)
  &\xrightarrow{d}
  N\!\left(0,\;\mathcal I_{\bar n}(\theta_0^\star)^{-1}\right).
    \label{eq:prop6-retained}
\end{align}
\endgroup
}
\end{prop}
\begin{proof}
See Appendix~\ref{app:proof-asnorm}.
\end{proof}
{Proposition~\ref{prop:consistency}(ii) permits $c\in(1/2,1]$ for consistency.
Proposition~\ref{prop:asnorm} requires $c<1$, so that $t\gamma_t\to\infty$ and
its limiting covariance is independent of $\gamma$.
At $c=1$, root-$T$ normality may still hold under an
additional condition such as $\gamma\lambda_{\min}(\mathcal H)>1/2$, but
the limiting covariance generally depends on $\gamma$.}

Since $N_T/T\to\pi_{\bar n}(\theta_0)$, the calendar-time and retained-sample
normalizations in \eqref{eq:prop6-collapsed}--\eqref{eq:prop6-retained} give
equivalent limits. MLE and ASGD have the
same first-order limit, although their finite-sample behavior may differ
substantially.

The truncated MLE is asymptotically efficient for the conditional experiment
based on the retained observations. For a fixed cutoff, its calendar-time
covariance is:
\[
  V_{\bar n}^{(T)}
  =
  [\pi_{\bar n}(\theta_0)
  \mathcal I_{\bar n}(\theta_0^\star)]^{-1}.
\]
If every cutoff is evaluated at the same regular representative
$\theta_0^\star=\theta_0$ and in the same coordinates, then
$V_{\bar n}^{(T)}$ is nonincreasing in $\bar n$ under the Loewner order.

If an estimator does not converge to a single representative, distributional
inference must also account for representative selection. For fixed algorithmic
choices and a specified data-generating parameter, the finite-sample
distribution of ASGD can nevertheless be studied by simulation.

\subsection{Score inference}

\paragraph{Regular points.}
For an interior parameter value $\theta$, define the retained score as:
\begin{equation}\label{eq:prop6-score-vector}
  U_T(\theta)
  =
  \sum_{t=1}^T
  \1\{S_t\in\cS_{\bar n}\}\,
  \nabla_\theta\log q(S_t;\theta).
\end{equation}
If $\mathcal I_{\bar n}(\theta_0)$ is nonsingular under the fixed interior
null hypothesis $\theta=\theta_0$, then:
\begin{equation}\label{eq:prop6-score}
  LM_T(\bar n;\theta_0)
  =
  \frac{1}{N_T}\,
  U_T(\theta_0)'\,
  \mathcal I_{\bar n}(\theta_0)^{-1}\,
  U_T(\theta_0)
  \xrightarrow{d}
  \chi^2_{\dim\Theta}.
\end{equation}
Unlike a Wald statistic, \eqref{eq:prop6-score} evaluates the score and
information at the null and requires no unrestricted estimator. Inverting the
test gives the level-$(1-\alpha)$ score confidence region:
\begin{equation}\label{eq:prop6-score-region}
  \mathcal C_{1-\alpha}
  =
  \left\{
    \theta\in\Theta:
    LM_T(\bar n;\theta)
    \leq \chi^2_{\dim\Theta,\,1-\alpha}
  \right\}.
\end{equation}

\paragraph{Constant-rank partial identification.}
If $\mathcal I_{\bar n}(\theta_0)$ is singular but has locally constant rank
$r_{\bar n}$ near $\theta_0$, inference can instead be restricted to the
first-order score directions:
\begin{equation}\label{eq:prop6-score-partial}
  LM_T^+(\bar n;\theta_0)
  =
  \frac{1}{N_T}\,
  U_T(\theta_0)'\,
  \mathcal I_{\bar n}(\theta_0)^+\,
  U_T(\theta_0)
  \xrightarrow{d}
  \chi^2_{r_{\bar n}},
\end{equation}
where $A^+$ is the Moore--Penrose pseudoinverse and
$r_{\bar n}=\operatorname{rank}\mathcal I_{\bar n}(\theta_0)$. Under the
locally constant-rank assumption, the constant-rank theorem gives:
\[
  \operatorname{Null}\mathcal I_{\bar n}(\theta_0)
  =
  T_{\theta_0}\Theta_0(\bar n),
  \qquad
  r_{\bar n}
  =\dim\Theta-\dim\Theta_0(\bar n).
\]

\paragraph{Rank-changing singularities.}
The preceding result does not cover points such as $A_0=0$ in {$\mathcal L_\Lambda$}. As discussed after
Proposition~\ref{prop:consistency}, every skew-direction score vanishes at
$A=0$, so the conditional information is singular. The first-order score
therefore contains no information about the skew coordinates. Inference for
these higher-order-identified directions would require a nonregular analysis, even
when a score statistic for the remaining first-order directions has a
chi-square limit.

\section{Monte Carlo evidence}
\label{sec:simulation-study}

This section reports Monte Carlo evidence on the identification and estimation
results developed above. The first experiment examines the finite
sign-and-transpose equivalence class, diagonal-similarity underidentification,
estimation of identified coordinates, and constant-rank score inference. The fixed-kernel experiments use
a regular point-identified design to study truncated maximum likelihood,
retained-sample score inference, and averaged stochastic gradient descent
(ASGD). A matched experiment in Online
Appendix~\ref{sec:simulation-contextual-asgd} considers ASGD when observed
covariates make the DPP kernel specific to each observation.
Additional Monte Carlo results for a stabilized ASGD specification are
reported in Online Appendix~\ref{app:stabilized-asgd}.

\subsection{Experiment 1: Identification and estimation of identified coordinates}
\label{sec:simulation-identification}

This experiment first enumerates the finite sign-and-transpose equivalence
class in Proposition~\ref{lda} and verifies the proposed normalizations. It
illustrates the continuous diagonal-similarity underidentification and
evaluates the estimation of the identified coordinates. We set \(n=4\), define
the design points as:
\[
  x=(-1.5,-0.5,0.5,1.5)',
\]
and let \(\bar A\) be skew-symmetric with upper-triangular entries:
\[
  (\bar a_{12},\bar a_{13},\bar a_{14},
    \bar a_{23},\bar a_{24},\bar a_{34})
  =(0.30,0.20,0.10,0.25,-0.15,0.35).
\]
Set:
\[
  L_0=\operatorname{diag}\{\exp(0.2x_i)\}_{i=1}^4+1.5\bar A.
\]
The kernel has no zero off-diagonal entries, and the smallest absolute determinant of
the three \(2\)-by-\(2\) cross-blocks in \eqref{simplified} is \(0.12375\).
Enumerating the kernel sets:
\[
  DL_0D^{-1}\quad\text{and}\quad DL_0'D^{-1},
  \qquad
  D=\operatorname{diag}(1,d_2,d_3,d_4),\quad d_j\in\{-1,1\},
\]
produces \(16=2^4\) distinct kernels. The computed principal minors and
subset probabilities agree across all members of this class to machine precision.
Imposing \(L_{1j}>0\) leaves one representative in each of the two kernel sets,
one generated from \(L_0\) and the other from \(L_0'\);
adding the condition \(L_{23}>0\) leaves one overall. Thus the
normalizations in Proposition~\ref{lda} reduce the computed class
\(16\longrightarrow2\longrightarrow1\).

To illustrate the continuous underidentification, we use the parameterization:
\begin{equation}
  \begin{aligned}
    L(\beta,\delta)
    &=
    B(\delta)
    \left[
      \operatorname{diag}\{\exp(\alpha_0+\alpha_1x_i)\}_{i=1}^4
      +e^\eta\bar A
    \right]
    B(\delta)^{-1}, \\
    B(\delta)
    &=
    \operatorname{diag}(1,e^{\delta_2},e^{\delta_3},e^{\delta_4}),
  \end{aligned}
  \label{eq:simulation-similarity-orbit}
\end{equation}
where \(\beta=(\alpha_0,\alpha_1,\eta)'\) and
\(\delta\in\Delta=[-0.03,0.03]^3\). For fixed \(\beta\), varying \(\delta\)
gives a three-dimensional family of observationally equivalent kernels.
At \(\beta=\beta_0\), this family contains \(L_0\). Because
diagonal similarity preserves every
principal minor, the full and truncated likelihoods are independent of
\(\delta\). We set
\(\beta_0=(0,0.2,\log1.5)'\) and restrict \(\beta\) to a compact box.\footnote{The box is
\([-0.75,0.75]\times[-0.30,0.70]\times[\log0.3,\log2.5]\).
A Gershgorin bound gives a uniform lower eigenvalue bound \(0.0452\) for
the symmetric part of \(L(\beta,\delta)\), so the entire experimental
parameter space remains in {\(\mathcal L_\Sigma\)}.}

In this identification experiment, we retain sets with
\(|S|\leq2\), for which the exact retention probability is \(0.61336\).
{Because diagonal similarity preserves principal minors,
Remark~\ref{rem:truncated-id} applies: the retained singleton and pair
configurations identify \(\lambda_i=\exp(\alpha_0+\alpha_1x_i)\) and, since
\(\bar A\) is fixed and \(\rho=e^\eta>0\), also \(\rho\).}
Hence
within this six-parameter experimental family:
\[
  \Theta_0(2)=\{\beta_0\}\times\Delta,
  \qquad
  \widehat\Theta_T(2)=\{\widehat\beta_T\}\times\Delta.
\]
Because the likelihood is constant in \(\delta\in\Delta\), the estimation accuracy is
determined entirely by the identified coordinates \(\beta\). We summarize it
using:
\[
  e_T=\|\widehat\beta_T-\beta_0\|,
\]
where \(\|\cdot\|\) denotes the Euclidean norm. Across \(N_{\mathrm{MC}}\) Monte Carlo
replications, the corresponding vector root mean squared error (RMSE) is:
\[
  \operatorname{RMSE}(\widehat\beta_T)
  =
  \left\{
    \frac{1}{N_{\mathrm{MC}}}\sum_{m=1}^{N_{\mathrm{MC}}}
    \|\widehat\beta_{T,m}-\beta_0\|^2
  \right\}^{1/2}.
\]
The six-dimensional conditional information matrix has locally constant rank three,
with three zero eigenvalues corresponding to the unidentified coordinates and
three positive eigenvalues corresponding to the identified coordinates.

\begin{table}[htbp]
  \centering
  \small
  \begin{tabular}{lrrrrr}
    \hline
    \(T\) & Mean \(N_T\) & Median \(e_T\) & \(\operatorname{RMSE}(\widehat\beta_T)\)
      & \(\chi^2_3\) size & \(\chi^2_6\) size \\
    \hline
      500 &  306.9 & 0.327 & 0.753 & 0.050 & 0.007 \\
     2000 & 1226.9 & 0.172 & 0.390 & 0.051 & 0.004 \\
     8000 & 4906.5 & 0.079 & 0.126 & 0.054 & 0.007 \\
    \hline
  \end{tabular}
  \caption{Estimation of identified coordinates and rank-score inference.
  The table uses 2,000 MLE replications and 5,000 independent score replications
  per row. The penultimate column uses the critical value with three degrees of
  freedom, equal to the information rank; the final column uses six degrees of
  freedom, equal to the ambient parameter dimension.}
  \label{tab:simulation-identification}
\end{table}

Table~\ref{tab:simulation-identification} reports estimation accuracy for the
identified coordinates and score-test size. Numerical optimization succeeds in every replication. The compact-box
constraint on \(\beta\) is active in \(17.9\%\), \(2.45\%\), and \(0.0\%\) of
the MLE replications as \(T\) increases. The estimation error in the identified
coordinates declines with \(T\), while the likelihood leaves the unidentified
coordinates unrestricted. The rank-three score test has size close to its
nominal level. Using the six-dimensional ambient parameter space instead
produces a severely conservative test.

To illustrate representative selection, we apply ASGD to a common
\(T=64{,}000\) sequence using one \(\beta\) start and three \(\delta\) starts.
The identified coordinates move toward their true values, while the
unidentified coordinates remain at their starting values. Different starts
therefore produce different kernels with numerically identical
subset probabilities.

\subsection{Experiment 2: Point-identified simulation designs}
\label{sec:simulation-design}
\label{sec:simulation-mle-score}

We next consider regular, point-identified designs that illustrate the
consistency and asymptotic-normality results in
Propositions~\ref{prop:consistency} and \ref{prop:asnorm}.
We set \(n=8\) and consider three truncation thresholds
\(\bar n\in\{4,5,8\}\), where \(\bar n=8\) corresponds to the full likelihood.
The data-generating process is a
three-parameter submodel of \eqref{qualitydiversity}. Define the standardized
covariate values as:
\[
  x_i
  =
  \frac{i-(n+1)/2}{\sqrt{(n^2-1)/12}},
  \qquad i=1,\ldots,n,
\]
and define the skew-symmetric matrix \(\widetilde R_0\) by:
\[
  (\widetilde R_0)_{ij}
  =
  \begin{cases}
    \operatorname{sign}(j-i)(1/2)^{|i-j|-1},&i\ne j,\\
    0,&i=j.
  \end{cases}.
\]
We set \(R_0=\widetilde R_0/\|\widetilde R_0\|_2\) and use:
\begin{equation}
  L(\theta)
  =
  D(\alpha)(I_n+\rho R_0)D(\alpha),
  \qquad
  D(\alpha)
  =
  \operatorname{diag}
  \left\{\exp\left(\frac{\alpha_0+\alpha_1x_i}{2}\right)\right\}_{i=1}^n,
  \label{eq:simulation-kernel}
\end{equation}
where \(\theta=(\alpha_0,\alpha_1,\eta)'\) and
\(\rho=\exp(\eta)>0\). The restriction \(\rho>0\) selects one of the two
observationally equivalent sign choices.

We fix \(\alpha_{1,0}=0.4\) and calibrate \(\alpha_{0,0}\) to the target
expected cardinality. For subset \(s\):
\[
  \det L_s
  =
  \exp\left\{\alpha_0|s|+\alpha_1\sum_{i\in s}x_i\right\}
  \det(I_n+\rho R_0)_s.
\]
The baseline design sets \(\rho_0=0.60\) and \(\E|S|=2.5\). The dense design
increases expected cardinality while holding dependence fixed, whereas the
weak-dependence design reduces \(\rho_0\) while approximately preserving the
baseline cardinality distribution. The baseline and dense designs have the same
\(\alpha_1\) and \(\rho\), but the dense value of \(\alpha_0\) is larger by
\(1.049\). Holding the other weight components fixed, this difference
multiplies the relative weight of each additional selected item by
\(\exp(1.049)\simeq2.85\), shifting probability toward larger sets. This is
the source of the lower retention probability in the dense design at a common
cardinality threshold. The weak-dependence design sets \(\rho_0=0.15\),
placing it near independence. Table~\ref{tab:simulation-dgp} reports the exact retention
probabilities \(\pi_{\bar n}=\Pr_{\theta_0}\{|S|\leq\bar n\}\).

\begin{table}[htbp]
  \centering
  \small
  \begin{tabular}{lrrrrr}
    \hline
    Design & \(\alpha_{0,0}\) & \(\rho_0\)
           & \(\E|S|\) & \(\pi_4\) & \(\pi_5\) \\
    \hline
    Baseline & \(-0.863\) & \(0.60\) & \(2.5\) & \(0.933\) & \(0.986\) \\
    Dense    & \( 0.186\) & \(0.60\) & \(4.5\) & \(0.494\) & \(0.759\) \\
    Weak     & \(-0.821\) & \(0.15\) & \(2.5\) & \(0.935\) & \(0.987\) \\
    \hline
  \end{tabular}
  \caption{Data-generating processes and exact retention probabilities}
  \label{tab:simulation-dgp}
\end{table}

For \(n=8\), all \(2^8=256\) subset probabilities and all conditional
information matrices are evaluated exactly. The main experiment uses
\(T\in\{250,1{,}000,4{,}000\}\) with 2,000 replications per cell. The design includes moderate truncation,
severe truncation in the dense case, and the full likelihood. Within each replication,
a single length-\(4{,}000\) sequence generates the nested samples for the three
values of \(T\) and all truncation levels. The score experiment uses an
independent set of 2,000 replications per cell.

An independent sample of 250 observations supplies the starting value and is
excluded from the reported \(T\). We initialize the quality coefficients by
regressing log singleton-to-empty-set frequency ratios on the fixed covariate
and estimate the dependence magnitude from pair frequencies, adding a one-half
pseudocount to each configuration count.\footnote{For the dependence parameter,
we use the upper median of the positive pair-based estimates of \(\rho^2\),
or \(\rho^2=0.04\) if none is positive.}
Numerical optimization and projected ASGD use
\((\alpha_0,\alpha_1,\eta)\in[-8,8]\times[-4,4]\times[-20,3]\).
For MLE, the starting value is
one of three numerical starts; the optimization uses the exact truncated score and
retains the solution with the highest objective. For ASGD, it is the initial
iterate. Each ASGD run is streaming and single-pass: \(T\) counts calendar
observations, \(N_T\) counts retained score updates, and each observation is
used once.

Table~\ref{tab:simulation-mle-score} compares truncated MLE accuracy,
asymptotic efficiency, and retained-sample score-test size in the regular
point-identified setting. Trace efficiency
is the ratio of the traces of the full- and truncated-likelihood asymptotic
covariance matrices, both on the calendar-time scale in
\eqref{eq:prop6-collapsed}. It captures both retention and conditional
information per retained observation. Truncation is most costly in the dense
design, where lower retention is accompanied by larger efficiency losses and
estimation error.

\begin{table}[htbp]
  \centering
  \small
  \begin{tabular}{lcrrrrr}
    \hline
    Design & \(\bar n\) & Mean \(N_T\)
      & RMSE\((\widehat\rho)\)
      & Trace efficiency
      & Regular share
      & Score size at \(5\%\) \\
    \hline
    Baseline & 4 & 3730.0 & 0.180 & 0.745 & 0.985 & 0.060 \\
             & 5 & 3945.5 & 0.159 & 0.902 & 0.989 & 0.058 \\
             & 8 & 4000.0 & 0.150 & 1.000 & 0.993 & 0.058 \\
    \hline
    Dense    & 4 & 1976.0 & 0.216 & 0.417 & 0.959 & 0.044 \\
             & 5 & 3037.9 & 0.163 & 0.690 & 0.988 & 0.045 \\
             & 8 & 4000.0 & 0.133 & 1.000 & 0.997 & 0.040 \\
    \hline
    Weak     & 4 & 3740.3 & 0.213 & 0.703 & 0.552 & 0.051 \\
             & 5 & 3949.6 & 0.202 & 0.882 & 0.559 & 0.053 \\
             & 8 & 4000.0 & 0.194 & 1.000 & 0.572 & 0.059 \\
    \hline
  \end{tabular}
  \caption{MLE accuracy and score size at \(T=4{,}000\). There are 2,000 MLE
  replications and 2,000 independent score replications per cell. Regular
  share is the proportion of MLE replications with an interior estimate and
  invertible plug-in information. The score statistic uses exact null
  information.}
  \label{tab:simulation-mle-score}
\end{table}

All MLE optimizations converge. In the weak-dependence design, estimates often
lie near independence, making the plug-in information numerically
singular.

For the interior simple null hypotheses considered here, the score statistic in
\eqref{eq:prop6-score} is available in every replication. Across the 27
combinations of design, \(T\), and \(\bar n\), rejection frequencies for the
\(N_T\)-scaled statistic range from \(0.0400\) to \(0.0595\), remaining close
to the nominal \(5\%\) level. These
frequencies support the realized-\(N_T\) normalization in
\eqref{eq:prop6-retained} and \eqref{eq:prop6-score}.

To examine the effect of normalization, let
\(\bar u_T=U_T(\theta_0)/N_T\) denote the average retained score. We compare
the statistic in \eqref{eq:prop6-score}:
\[
  Q_{N_T}
  =
  N_T\bar u_T'\mathcal I_{\bar n}(\theta_0)^{-1}\bar u_T
  =LM_T(\bar n;\theta_0),
\]
with its calendar-time counterpart:
\[
  Q_T^{\mathrm{cal}}
  =
  T\bar u_T'\mathcal I_{\bar n}(\theta_0)^{-1}\bar u_T
  =
  \frac{T}{N_T}Q_{N_T},
\]
which replaces the retained sample size by the calendar sample size.
Table~\ref{tab:simulation-score-scaling} uses the same score and exact null
information in both columns.

\begin{table}[htbp]
  \centering
  \small
  \begin{tabular}{crrr}
    \hline
    \(\bar n\) & Mean \(N_T\) & Size of \(Q_{N_T}\) & Size of \(Q_T^{\mathrm{cal}}\) \\
    \hline
    4 & 1975.2 & 0.044 & 0.267 \\
    5 & 3037.5 & 0.045 & 0.107 \\
    8 & 4000.0 & 0.040 & 0.040 \\
    \hline
  \end{tabular}
  \caption{Effective-sample-size normalization in the dense design at
  \(T=4{,}000\). Each row uses 2,000 independent replications and the \(5\%\)
  \(\chi^2_3\) critical value; all replications enter.}
  \label{tab:simulation-score-scaling}
\end{table}

The plug-in MLE Wald results display the same normalization effect. In the
dense design, rejection frequencies under the \(N_T\) and \(T\) multipliers
are \(0.066\) and \(0.265\), respectively, at \(\bar n=4\), and \(0.088\) and
\(0.142\) at \(\bar n=5\). These frequencies are conditional on an interior
estimate and invertible plug-in information, whereas the score frequencies in
Table~\ref{tab:simulation-score-scaling} are unconditional.

Exact enumeration of the \(2^8\) subsets confirms the covariance ordering
following \eqref{eq:prop6-collapsed}. In all three designs, both
\(V_4^{(T)}-V_5^{(T)}\) and \(V_5^{(T)}-V_8^{(T)}\) are positive definite;
their smallest eigenvalues range from \(0.00516\) to \(0.35012\). Thus
increasing \(\bar n\) reduces asymptotic covariance in every parameter
direction, accounting jointly for retention and information per retained
observation.

\subsection{Experiment 3: The choice of learning rate and the initialization}
\label{sec:simulation-asgd}

In the ASGD algorithm, the score is multiplied by the learning rate
\(\gamma t^{-0.75}\), and \(t\) advances with every calendar observation,
including observations that are not retained. The parameter estimate is updated only when
the observation is retained, and averaging begins with the first calendar
iterate. We calibrate \(\gamma\) separately for each design and truncation
level using 100 independent pilot paths of length \(T=1{,}000\) before
comparing ASGD with MLE. Because the
data-generating parameter is known, we select the \(\gamma\) candidate with the smallest
Fisher-weighted squared error among those for which at most \(5\%\) of paths
reach the parameter bounds. The selected learning rate is held fixed at all
reported simulation horizons. In the
2,000-replication main experiment at \(T=4{,}000\), the ASGD-to-MLE
\(\rho\)-RMSE ratio ranges from \(2.23\) to \(10.12\).

To assess the persistence of initialization effects, we extend the
recursion to:
\[
  T\in\{4{,}000,16{,}000,64{,}000,256{,}000\},
  \qquad
  \bar n\in\{4,5,8\},
\]
using 200 paired replications. Table~\ref{tab:simulation-asgd-long} reports
the largest-horizon results for the baseline and dense designs. Despite the
large number of updates, ASGD remains substantially less accurate than MLE,
including without truncation. Its RMSE declines more slowly over the reported
horizons under the fixed learning-rate schedule. A substantial accuracy gap
also remains in the weak-dependence design.\footnote{At \(T=256{,}000\),
ASGD \(\rho\)-RMSE is \(5.73\)--\(6.56\) times MLE RMSE in the
weak-dependence design.}

\begin{table}[htbp]
  \centering
  \small
  \begin{tabular}{lcrrrr}
    \hline
    Design & \(\bar n\) & Mean \(N_T\)
      & MLE RMSE & ASGD RMSE & ASGD/MLE \\
    \hline
    Baseline & 4 & 238717 & 0.0221 & 0.2152 &  9.75 \\
             & 5 & 252499 & 0.0200 & 0.2222 & 11.10 \\
             & 8 & 256000 & 0.0186 & 0.2712 & 14.60 \\
    \hline
    Dense    & 4 & 126386 & 0.0231 & 0.4725 & 20.42 \\
             & 5 & 194360 & 0.0185 & 0.2546 & 13.75 \\
             & 8 & 256000 & 0.0158 & 0.5522 & 34.87 \\
    \hline
  \end{tabular}
  \caption{Calendar-time-indexed ASGD versus paired MLE at
  \(T=256{,}000\),
  based on 200 paired replications. RMSE refers to \(\widehat\rho\).}
  \label{tab:simulation-asgd-long}
\end{table}

A local calculation relates the persistence of initialization effects to
the curvature. Let \(e_t=\hat\theta_t-\theta_0\) denote the raw-iterate
error and let
\(\mathcal H=\pi_{\bar n}(\theta_0)
\mathcal I_{\bar n}(\theta_0)\) be calendar-time information. Linearizing
the mean score gives:
\begin{equation}
  e_t
  \simeq
  \{I_3-\gamma t^{-c}\mathcal H\}e_{t-1}
  +\gamma t^{-c}\xi_t,
  \label{eq:simulation-local-asgd}
\end{equation}
where \(\xi_t\) is the centered score innovation. Diagonalizing the
positive-definite matrix \(\mathcal H\) separates fast and slow local
directions. Its spectral condition number is:
\[
  \kappa(\mathcal H)
  =\frac{\lambda_{\max}(\mathcal H)}{\lambda_{\min}(\mathcal H)}.
\]
In the simulation's \((\alpha_0,\alpha_1,\eta)\) coordinates,
\(\kappa(\mathcal H)\) is
approximately \(377\)--\(470\) in the regular baseline and dense
\(\bar n\in\{4,5\}\) cells. With the selected learning rate and \(c=0.75\), the
linearized persistence factor along the slowest direction:
\[
  \exp\left\{-\gamma\lambda_{\min}(\mathcal H)
  \frac{T^{1-c}-1}{1-c}\right\},
\]
is \(0.63\)--\(0.81\) at \(T=256{,}000\). In the weak-dependence cells,
\(\kappa(\mathcal H)\) is about \(79{,}000\)--\(89{,}000\), and the factor is about
\(0.998\). This coordinate-dependent factor measures the local contraction of the
initialization component along the slowest direction. Total error for the
averaged estimator also reflects accumulated score innovations and averaging.

The sensitivity with respect to the learning rate supports this interpretation.
We take \(c=0.75\), the
midpoint of the theoretically admissible interval \((1/2,1)\), and examine
\(c=0.60\) as a slower-decay alternative. Holding \(\gamma\) fixed, this change reduces
the \(\rho\)-RMSE by \(17.5\%\)--\(74.0\%\) across the nine cells at
\(T=256{,}000\), with RMSE remaining \(3.55\)--\(12.04\) times the MLE value.
Changing \(\gamma\) affects both the deterministic contraction and score
noise. In the separate 500-replication \(T=4{,}000\) sensitivity experiment,
replacing the selected \(\gamma\) by \(\gamma/4\) increases the \(\rho\)-RMSE in all
nine cells. Replacing the selected \(\gamma\) by \(4\gamma\) improves the
\(\rho\) component in the dense and
weak-dependence designs, but worsens it in the baseline design, and
\(33.2\%\)--\(95.4\%\) of
paths activate the projection. Larger cumulative learning rates therefore reduce
initialization bias, but increase projection activity and score noise. No
learning rate dominates across the three designs.

\section{Applications to cyber-event data}
\label{sec:application}

This section considers two cyber-event datasets with different scales
and sampling frequencies. The first records weekly combinations of breach types
and economic sectors over 772 weeks, providing a relatively small, low-frequency
application for model comparison and likelihood-based estimation. The second
records authentication activity every second for ten computers over 58 days,
yielding more than five million observations and illustrating one-pass ASGD in a
high-frequency data stream.
See Online Appendix~\ref{sec:simulation-contextual-asgd} for Monte Carlo
evidence on estimation with large $T$.

\subsection{Privacy Rights Clearinghouse sector--breach data}
\label{sec:cyber-eling-jung}
{Cyber losses differ by sector and breach type, and an insurer
writing cyber policies diversifies its portfolio across sectors. Whether this
diversification reduces risk exposure depends on whether breaches co-occur
across sectors and breach types in the same week. We study this
contemporaneous co-occurrence} using the Privacy Rights
Clearinghouse (PRC) database, which is widely used in the cyber-risk literature
\citep{eling2018copula,lu2024cyber}. We aggregate the records into $T=772$
consecutive weeks from 10 January 2005 through 27 October 2019. Following
the cross-sectional classification of \citet{eling2018copula}, we represent
weekly breaches by $X_t=(X_{t,gb})\in\{0,1\}^{5\times4}$. Rows index five
sectors---financial and insurance services (BSF), retail and other business
($\mathrm{BUS}=\mathrm{BSR}\cup\mathrm{BSO}$), medical services (MED),
education (EDU), and government and nonprofit organizations
($\mathrm{PUBLIC}=\mathrm{GOV}\cup\mathrm{NGO}$)---and columns index four
breach types: hacking (HACK), electronic-device loss
($\mathrm{ELET}=\mathrm{PHYS}\cup\mathrm{PORT}\cup\mathrm{STAT}$), unintended
disclosure (DISC), and insider breach (INSD). We set $X_{t,gb}=1$ if at least
one breach of type $b$ occurs in sector $g$ during week $t$, and zero otherwise.
Vectorizing $X_t$ gives the $n=20$ binary outcomes used in the analysis.
The corresponding random set is $S_t=\{(g,b):X_{t,gb}=1\}$.
For example, the first observed set corresponds to the week beginning January~10,~2005:
\begin{equation*}
X_1=
\begin{array}{c|rrrr}
 & \mathrm{HACK} & \mathrm{ELET} & \mathrm{DISC} & \mathrm{INSD} \\
\hline
\mathrm{BSF} & 0 & 0 & 0 & 0 \\
\mathrm{BUS} & 0 & 0 & 0 & 0 \\
\mathrm{MED} & 0 & 0 & 0 & 0 \\
\mathrm{EDU} & 1 & 0 & 0 & 0 \\
\mathrm{PUBLIC} & 0 & 0 & 0 & 0 \\
\end{array},
\qquad S_1=\{(\mathrm{EDU},\mathrm{HACK})\}.
\end{equation*}
This week contains an education-sector hacking incident and no incident in the other 19 sector--type cells. Over the $T$ weeks, observed cardinalities range from 0 to 14, with an average of 5.418. All 190 pairs of cells co-occur at least once.
{The sequence $(X_t)_{t=1}^T$ is a matrix-valued binary time series.
Cells in the same row share a sector and cells in the same column share a
breach type, so dependence is organized by row and column membership.}

\subsubsection{Model specification}
\label{sec:cyber-model-specification}
Previous studies report time trends in these data \citep{lu2024cyber}. We
therefore use a time-varying version of the factorization in
\eqref{qualitydiversity}:
\begin{equation}
		\label{dt}
L_t=D_t(I_{20}+R)D_t
=\Lambda_t+D_tRD_t,
\end{equation}
where the diagonal matrices are:
\[
D_t=\operatorname{diag}
\bigl(\sqrt{\lambda_{1t}},\ldots,\sqrt{\lambda_{20,t}}\bigr),
\qquad
\Lambda_t=D_t^2,
\]
and $R'=-R$. We allow $\Lambda_t$ to vary over time, but hold $R$ fixed.
Proposition~\ref{prop:marginal-parameterization} shows that, for fixed $R$,
an unrestricted positive diagonal matrix can attain any vector of marginal inclusion
probabilities. We parameterize its time path by:
\[
\lambda_{i,t}=\exp(z_t'\beta_i), \qquad i=(g,b),
\]
where
\(z_t=(1,b_1(\tau_t),b_2(\tau_t),b_3(\tau_t))'\in\mathbb R^4\)
collects the basis functions of a cubic spline in normalized time
\(\tau_t=(t-1)/(T-1)\), and $\beta_i$ is a four-dimensional vector of
coefficients for cell $i=(g,b)$.

An unrestricted $R$ has $20(20-1)/2=190$ free entries. To reduce
dimensionality while preserving sector and breach-type structure, we consider
four restrictions:
\begin{align}
		\label{benchmark}
		R&=S\otimes I_4+I_5\otimes B \qquad \text{(axis model)},\\
		R&=S\otimes I_4+C_S\otimes B \qquad \text{(breach-led model)},\label{eq:cyber-final-selected-model}
		\\
			R&=S\otimes C_B+I_5\otimes B  \qquad \text{(sector-led model)},	\label{symmetric2}
		\\
				R&=S\otimes C_B+C_S\otimes B \qquad \text{(two-way model)}.		\label{eq:cyber-two-way}
\end{align}
{Here $\otimes$ denotes the Kronecker product, $S$ ($5\times5$) and $B$
($4\times4$) are skew-symmetric sector and breach-type matrices, and $C_S$
and $C_B$ are symmetric loading matrices with unit
diagonals.\footnote{One diagonal restriction would suffice to fix
the scale between $C_S$ and $B$ (or between $S$ and $C_B$); fixing the whole
diagonal is a substantive homogeneity restriction. Neither matrix needs to be
a correlation matrix.}
For cells $i=(g,b)$ and $j=(h,c)$, the axis model implies:
\begin{equation}
r_{(g,b),(h,c)}
=S_{gh}\mathbf 1\{b=c\}+\mathbf 1\{g=h\}B_{bc}.
\label{eq:cyber-type-compat-entry}
\end{equation}
Thus $B$ is the breach-type interaction common to all sectors, and $S$ is the
sector interaction common to all breach types. Two cells that share neither a
sector nor a breach type have $r_{ij}=0$, hence an odds ratio of one
conditional on all other cells being zero. The other three models add this
cross channel. Writing $H_S=C_S-I_5$, the breach-led model
\eqref{eq:cyber-final-selected-model} is:
\begin{equation}
	R=S\otimes I_4+I_5\otimes B+H_S\otimes B,
	\label{eq:cyber-retained-decomposition}
\end{equation}
so that, for two distinct cells:
\[
r_{(g,b),(h,c)}
=
\begin{cases}
	S_{gh}, & \text{ if } g\ne h,\ b=c,\\
	B_{bc}, & \text{ if } g=h,\ b\ne c,\\
	(H_S)_{gh}B_{bc}, & \text{ if } g\ne h,\ b\ne c.
\end{cases}
\]
The cross effect is the breach-pair pattern $B_{bc}$ scaled by a sector-pair
loading $(H_S)_{gh}$; the unit diagonal of $C_S$ keeps $B$ as the common
within-sector block. The sector-led model \eqref{symmetric2} is the mirror
image, with cross effect $S_{gh}(H_B)_{bc}$, where $H_B=C_B-I_4$, and the
two-way model \eqref{eq:cyber-two-way} contains both. The breach-led model
has 80 cubic spline coefficients and 26 dependence parameters: ten in $S$,
six in $B$, and ten off-diagonal entries in $C_S$.\par}

\subsubsection{Estimation and findings}
We estimate the four NDPP candidates \eqref{benchmark}--\eqref{eq:cyber-two-way},
the model with time-adjusted independence ($R=0$), and the symmetric-DPP
benchmark in the spirit of \citet{mariet2016kronecker} (see Online
Appendix~\ref{app:cyber-matrix-literature}). We use the weekly matrices,
maximize each exact full-sample likelihood, and compare
the models using the Akaike Information Criterion (AIC), Bayesian Information
Criterion (BIC), and a blocked
predictive log score, defined as the average held-out log probability assigned
to the weekly observed set. %

\begin{table}[!htbp]
\centering
\begin{singlespace}
\scriptsize
\setlength{\tabcolsep}{3.5pt}
\renewcommand{\arraystretch}{1.05}
\begin{tabular}{lrrrrr}
\hline
Specification & Dependence parameters & Log likelihood & $\Delta$AIC & $\Delta$BIC
& Maximum absolute score\\
\hline
Time-adjusted independence & 0  & -6710.677 & 0.0    & 0.0  & --\\
Symmetric Kronecker DPP     & 16 & -6710.677 & 32.0   & 106.4 & $1.22{\times}10^{-6}$\\
Axis                       & 16 & -6660.681 & -68.0  & 6.4  & $5.82{\times}10^{-6}$\\
Sector-led                 & 22 & -6637.049 & -103.3 & -1.0 & $8.81{\times}10^{-6}$\\
Breach-led                 & 26 & -6619.405 & -130.5 & -9.7 & $5.39{\times}10^{-6}$\\
Two-way                    & 32 & -6609.976 & -137.4 & 11.4 & $9.89{\times}10^{-6}$\\
\hline
\end{tabular}
\caption{Fit of the symmetric and nonsymmetric DPP specifications.
Information-criterion differences are relative to time-adjusted independence.}
\label{tab:cyber-retained-fit}
\end{singlespace}
\end{table}

The symmetric fit uses the independence start and 11 deterministic nonzero
starts. Every run converges to
$C_G^{\mathrm{sym}}=I_5,C_B^{\mathrm{sym}}=I_4$, that is, to time-adjusted
independence. The NDPP axis model improves the fit without adding parameters.
Its advantage comes from relaxing the restriction to nonpositive pairwise
dependence rather than increasing the model complexity.

Five-fold blocked cross-validation also places the symmetric DPP at the
independence boundary. The NDPP axis model improves the pooled held-out log
score by $0.0205$ per week, with positive gains in three of the five blocks.

{\emergencystretch=1em
The breach-led specification has lower AIC and BIC values than both the
time-adjusted independence model and the axis model, supporting the inclusion of cross-axis
dependence. AIC selects the two-way model, whereas BIC prefers the more parsimonious breach-led
model in \eqref{eq:cyber-final-selected-model}. We retain the latter one as the
primary specification.\par}

{\emergencystretch=2em
{Sign similarity and transposition preserve every principal minor
(Proposition~\ref{lda}) and hence leave the likelihood unchanged. Within the
breach-led class, the sign changes that preserve the Kronecker structure are
of the form $E_S\otimes E_B$, where $E_S$ and $E_B$ are $5\times5$ and
$4\times4$ diagonal sign matrices. They map $(S,B,C_S)$ to
$(E_SSE_S,E_BBE_B,E_SC_SE_S)$, whereas transposition maps $(S,B)$ to
$(-S,-B)$. Because the first-row normalization of Proposition~\ref{lda} need
not preserve this structure, we select, among these equivalent solutions, the
lexicographically largest coefficient vector, listing the upper-triangular
entries of $S$, $B$, and $C_S$ in row order.}
Table~\ref{tab:cyber-selfscaled-parameters} in
Online Appendix~\ref{app:cyber-selfscaled-parameters} reports the resulting estimates.
The signs select a reproducible solution, but are not themselves
identified. The magnitudes and fitted probabilities are invariant to that
choice.\par}

Table~\ref{tab:cyber-marginal-correlations} reports the five largest positive
fitted correlations after pooling the weeks with equal weights. Pooling
includes both average within-week covariance and covariance between the weekly
marginal probabilities. Specifically, let:
\[
p_{it}=(K_t)_{ii},\qquad
c_{ij,t}=-(K_t)_{ij}(K_t)_{ji},\qquad
\bar p_i=\frac1T\sum_{t=1}^T p_{it}.
\]
The pooled covariance is
\[
\operatorname{Cov}_{\mathrm{pooled}}(X_i,X_j)
=\frac1T\sum_{t=1}^T c_{ij,t}
+\frac1T\sum_{t=1}^T(p_{it}-\bar p_i)(p_{jt}-\bar p_j).
\]
Dividing by $\sqrt{\bar p_i(1-\bar p_i)\bar p_j(1-\bar p_j)}$ gives the
pooled correlation. Two of the five pairs differ in both sector and breach type.
\begin{table}[htbp]
	\centering
	\begin{tabular}{clcc}
		\hline
		Rank & Pair & Correlation & Pair type \\
		\hline
		1 & MED--HACK and MED--DISC       & 0.44799 & Breach type \\
		2 & BUS--HACK and MED--DISC       & 0.29953 & Cross       \\
		3 & BUS--HACK and MED--HACK       & 0.29064 & Sector      \\
		4 & BUS--HACK and MED--ELET       & 0.29055 & Cross       \\
		5 & EDU--ELET and PUBLIC--ELET    & 0.28169 & Sector      \\
		\hline
	\end{tabular}
	\caption{Five largest fitted correlations after pooling the 772 weeks.}
	\label{tab:cyber-marginal-correlations}
\end{table}

To isolate contemporaneous NDPP dependence, we instead compute correlations
conditional on each week:
\[
\rho_{ij,t}
=
\frac{c_{ij,t}}
{\sqrt{p_{it}(1-p_{it})p_{jt}(1-p_{jt})}},
\]
Their average $\bar\rho_{ij}=T^{-1}\sum_{t=1}^T\rho_{ij,t}$ measures the mean of
within-week correlation, excluding co-movements between the weekly marginal
probabilities.
Table~\ref{tab:cyber-pair-class-dependence} summarizes these correlations by
the three classes of pairs induced by \eqref{eq:cyber-retained-decomposition}.
\begin{table}[!htbp]
\centering
\begin{singlespace}
\scriptsize
\setlength{\tabcolsep}{3pt}
\renewcommand{\arraystretch}{1.08}
\begin{tabular}{@{}lrrrl@{}}
\hline
Pair class & Pairs & Mean $\bar\rho_{ij}$ & Mean $|\rho_{ij,t}|$
& Largest $\bar\rho_{ij}$\\
\hline
Same sector, different breach types & 30  & 0.0151 & 0.0158
& MED--DISC / MED--INSD (0.0354)\\
Different sectors, same breach type & 40  & 0.0285 & 0.0300
& BUS--ELET / MED--ELET (0.0733)\\
Different sectors, different breach types & 120 & 0.0114 & 0.0136
& BSF--INSD / EDU--HACK (0.0555)\\
\hline
\end{tabular}
\caption{Contemporaneous dependence under the selected breach-led model.
The mean absolute correlation averages $|\rho_{ij,t}|$ over the relevant
pairs and weeks.}
\label{tab:cyber-pair-class-dependence}
\end{singlespace}
\end{table}

Pairs from different sectors and the same breach type have the strongest
average dependence. The class comprising different sectors and different breach
types is weaker on average, but contains some comparably large associations.
Overall, 169 of the 190 fitted
$\bar\rho_{ij}$ are positive, highlighting the positive co-occurrence
allowed by the nonsymmetric DPP; a symmetric DPP would instead impose
nonpositive pairwise covariances. Moving from the re-estimated axis
specification to the selected breach-led specification raises the mean
absolute correlation in the class comprising different sectors and different
breach types from 0.0024 to 0.0136.

To measure each pair class's contribution to fluctuations in weekly
cardinality, we sum its conditional covariances. Each contribution therefore
reflects both the number of pairs and the magnitude of their covariances.
For the weekly cardinality $|S_t|$, the conditional variance admits the decomposition:
\[
\operatorname{Var}(|S_t|\mid z_t)
=
\sum_i p_{it}(1-p_{it})
+2\sum_{i<j}c_{ij,t}.
\]
The first term is the contribution of independent Bernoulli variables with the
same marginals. We split the second term among three pair classes: (i) same sector,
different breach types; (ii) different sectors, same breach type; and (iii)
different sectors, different breach types. Averaging over 772 weeks yields:
\[
\underbrace{3.74473}_{\text{selected NDPP}}
=
\underbrace{2.80760}_{\substack{\text{independence with}\\\text{the same marginals}}}
+\underbrace{0.13646}_{\substack{\text{same sector,}\\\text{different breach types}}}
+\underbrace{0.37287}_{\substack{\text{different sectors,}\\\text{same breach type}}}
+\underbrace{0.42780}_{\substack{\text{different sectors,}\\\text{different breach types}}}.
\]
Contemporaneous dependence raises the fitted average conditional cardinality
variance by $33.4\%$ relative to independence with the same marginals.
Different-sector, different-breach pairs make the largest aggregate
contribution despite their weaker average pairwise correlations, illustrating
why modest dependence across many pairs can matter in aggregate. This
decomposition attributes variance to cell-pair classes, not to individual
terms of $R$, because $K_t$ depends on the complete kernel.

\subsubsection{Alternative estimation procedures}
Let us now assess the selected breach-led specification under two alternative
estimation procedures: cardinality-truncated likelihood and stochastic-gradient
ascent.

\paragraph{Truncated likelihood.}
We reestimate the selected breach-led specification using $\bar n=8$, which
retains 86.7\% of the weeks. The fitted marginal probabilities remain close to
those from the full MLE. After applying the same normalization, the dependence
coefficients are also close and retain the same signs. On the common
untruncated criterion, the full MLE gains 22.910
log-likelihood units. In the predictive comparison using five contiguous blocks, it
gains 0.441 conditional log-score units per held-out retained week.

\paragraph{One-pass ASGD estimator.}
We implement the observation-level recursion in
Section~\ref{sec:inference}. We set $\bar n=20$, so ASGD and full MLE use
the same untruncated criterion. ASGD processes the 772 weekly pairs $(S_t,z_t)$
exactly once in chronological order using:
\[
\gamma_t=\gamma_0t^{-0.60},\qquad t=1,\ldots,772.
\]
Every post-update iterate enters the Polyak--Ruppert average. We consider two
starting values: the full-sample axis-model estimate embedded in the selected
model (axis warm start), and the full-sample cell-specific cubic spline estimates with
the upper entries of $S$ and $B$ initialized at $0.10$ (cellwise start).
{The skew entries start away from zero because the dependence
score vanishes at $R=0$, so an ASGD recursion started there never leaves
independence (see the discussion after
Proposition~\ref{prop:consistency}).}
Both
starts use the complete sample, so the exercise evaluates one ASGD pass
following initialization using the complete sample.

We use $\gamma_0=5$ as the baseline and also report the sensitivity to
$\gamma_0\in\{0.03,0.10,0.30,1,3,5,10\}$. Writing
$\hat\theta_T^{\mathrm{MLE}}$ for the full MLE and $\bar\theta_T$ for the
ASGD average, define the likelihood shortfall as:
\[
\Delta_\ell
=
\ell_T(\hat\theta_T^{\mathrm{MLE}})
-
\ell_T(\bar\theta_T).
\]

\begin{table}[!htbp]
\centering
\begin{singlespace}
\scriptsize
\setlength{\tabcolsep}{4pt}
\renewcommand{\arraystretch}{1.05}
\begin{tabular}{lrrrrr}
\hline
Initial value & Initial $\ell_T$ & Raw endpoint $\ell_T$
& Averaged $\ell_T$ & $\Delta_\ell$ & $\|\nabla\ell_T\|_\infty/T$\\
\hline
Axis warm start & -6660.681 & -6919.064 & -6748.920 & 129.515 & 0.0777\\
Cellwise start  & -6708.017 & -6931.639 & -6770.810 & 151.405 & 0.0559\\
\hline
\end{tabular}
\caption{Chronological one-pass ASGD results for $\gamma_0=5$. The last two
columns refer to the Polyak--Ruppert average; the deterministic full MLE has
$\ell_T=-6619.405$.}
\label{tab:cyber-one-pass-asgd}
\end{singlespace}
\end{table}

Table~\ref{tab:cyber-one-pass-asgd} shows substantial finite-sample
optimization error. For both starting values, averaging produces a higher
likelihood than the raw endpoint, but a lower likelihood than the starting
estimate. Among the seven values of the learning-rate constant $\gamma_0$
considered, the highest averaged log-likelihoods are $-6666.016$ for the
axis warm start and $-6713.656$ for the cellwise start. Both are attained
at $\gamma_0=0.03$, and neither exceeds the corresponding starting
likelihood. One pass
requires 772 component-score evaluations, equivalent to one full-data pass;
the update loops took 0.16 and 0.19 seconds, respectively, excluding
full-criterion diagnostics. At $T=772$, one-pass ASGD is fast, but produces
substantially lower likelihoods than full-sample likelihood maximization. This weekly exercise
evaluates finite-sample optimization on a
time-ordered path.

\newcommand{\LANLASGDReference}{Proposition~\ref{prop:asnorm}}

\subsection{LANL enterprise authentication data}
\label{sec:lanl-asgd-application}

{\emergencystretch=1em
The second application uses the authentication component of the Los Alamos National Laboratory
(LANL) Comprehensive, Multi-Source Cyber-Security Events database
\citep{kent2015comprehensive}, whose length and heterogeneity make one-pass
ASGD attractive. The authentication dataset contains 1,051,430,459
records involving 16,230 distinct source computers and 15,895 distinct
destination computers over 58 consecutive days. LANL reports elapsed seconds
from an undisclosed epoch, so we index the sample by data day rather than
calendar date. Of these records, 44.53\% have orientation
\texttt{LogOn}. Among the LogOn records, 78.94\% involve different source
and destination computers. We focus on these cross-computer LogOn events
because they involve the use of credentials to access another machine.
For example, one LogOn record is:\par}
\begin{singlespace}
\par\noindent\makebox[\textwidth][c]{%
\small\texttt{4,C553\$@DOM1,C553\$@DOM1,C553,C523,Kerberos,Network,LogOn,Success}.}
\par
\end{singlespace}
This record says that at second 4, source computer \texttt{C553} successfully
authenticated to destination computer \texttt{C523} through a network LogOn.
{For a set $\mathcal C$ of source computers, the observation at
second $t$ is the random set:
\begingroup
\setlength{\abovedisplayskip}{0pt}
\setlength{\abovedisplayshortskip}{0pt}
\setlength{\belowdisplayskip}{0pt}
\setlength{\belowdisplayshortskip}{0pt}
\begin{equation}
 \begin{aligned}
 S_t=\{i\in\mathcal C:\ &i\text{ initiates at least one cross-computer LogOn}\\
 &\text{authentication attempt during second }t\}.
 \end{aligned}
 \label{eq:lanl-random-set}
\end{equation}
\endgroup}%
To estimate a dense dependence model for
an informative subset of the 16,230 sources, many of which have sparse binary
activity series, we construct a ten-computer panel in two stages using
Days~1--8. On Day~1, we first retain the 20 most active sources, which
initiate 16.45\% of that day's cross-computer LogOn records.
Then we retain the ten candidates with the strongest aggregate positive same-second
residual dependence after accounting for their individual activity
{\interfootnotelinepenalty=10000
dynamics.\footnote{To select these ten source computers,
for each of the 20 source computers selected on Day~1, indexed by $i$, we fit, on the calibration seconds
$\mathcal T_c$ in Days~2--8, a marginal model corresponding to the NDPP
specification in \eqref{eq:lanl-dynamic-ndpp} below with $R=0$. Estimation uses
ridge-penalized maximum likelihood \citep[see][]{lecessie1992ridge}, with an
unpenalized intercept and penalty coefficient $10^{-7}$ on the remaining
coefficients. Let $\widehat p_{it}$ denote the
fitted conditional inclusion probability and form
$e_{it}=X_{it}-\widehat p_{it}$, observed activity minus its fitted probability.
Define:
\[
 \varrho_{ij}=\frac{\sum_{t\in\mathcal T_c}e_{it}e_{jt}}
 {\{\sum_{t\in\mathcal T_c}e_{it}^2
     \sum_{t\in\mathcal T_c}e_{jt}^2\}^{1/2}},
 \qquad
 q_i=\sum_{j\ne i}\max\{\varrho_{ij},0\}.
\]
The ten largest $q_i$ determine the panel. The quantity $\varrho_{ij}$ measures the
remaining same-second co-activity of source computers $i$ and $j$ after accounting
for their fitted individual activity probabilities, and $q_i$ totals source computer $i$'s
positive links to the others.}} We denote the selected ten-computer set by
$\mathcal C$. Its members are \texttt{C105}, \texttt{C1968}, \texttt{C1618},
\texttt{C527}, \texttt{C574}, \texttt{C539}, \texttt{C1619}, \texttt{C1795},
\texttt{C1115}, and \texttt{C553}.

In the example above, $\texttt{C553}\in S_4$.
Across the 58 days, the ten selected sources initiate 23,306,801 cross-computer LogOn
records, or 6.31\% of all such records. Collapsing these records to second-level
indicators produces 5,011,200 ten-dimensional observations. All $2^{10}$ set configurations occur at least
once.

For the LANL data, we extend the time-varying specification to capture intraday
seasonality, weekly patterns, and lagged activity:
\begin{equation}
 L_t=D_t(I_{10}+R)D_t,
 \qquad (D_t)_{ii}=\exp(z_{it}'\beta_i/2),
 \qquad R'=-R.
 \label{eq:lanl-dynamic-ndpp}
\end{equation}
The 17 components of $z_{it}$ comprise 14 time-based predictors and three
lagged-activity predictors. The time-based predictors are an intercept; a sine
and cosine for each of the 24- and 12-hour cycles; six indicators for the seven
positions in the weekly cycle; and three polynomial trend terms for gradual
change over the 58-day sample. The remaining predictors are whether source
computer $i$ was active in the previous second, its fraction of active seconds
during the previous minute, and the previous second's log-transformed number
of active sources, $\log(1+|S_{t-1}|)$.

The NDPP has $10\times17=170$ activity coefficients and
$\binom{10}{2}=45$ dependence coordinates, for a total of 215 parameters.
Setting $R=0$
yields a dynamic benchmark in which the computers' activities are conditionally independent
within each second given the time and lag predictors.

We use Days~1--8 to select the computer panel and determine starting values
and coordinate scales. Days~16--22 are used to tune ASGD and initialize the
main fit. We then make one chronological pass through Days~23--51 and
evaluate the fitted models on Days~52--58 with coefficients held fixed.
Panel~A of Table~\ref{tab:lanl-asgd-results} summarizes this chronology.

For both models, each ASGD update uses the full one-period likelihood. Indeed,
since the panel contains only ten computers, evaluating this likelihood is
computationally inexpensive, so no cardinality truncation is used.

We apply ASGD in scaled coordinates $v$, where
$\theta=\theta_{\mathrm{init}}+C_{\mathrm{sc}}v$.
The starting value $\theta_{\mathrm{init}}$ and fixed positive diagonal scaling
matrix $C_{\mathrm{sc}}$ are constructed from Days~2--8 as described in
Online Appendix~\ref{app:lanl-asgd-implementation}.
Within each path, reindex the observations by $m=1,\ldots,J$, and let $m_0$
denote the number of preceding warm-up updates carried into the step-size
clock. After observing $S_m$, the procedure takes a gradient step in $v$ using
step size $\gamma(m_0+m)^{-c}$. It then moves to the next second without
returning to earlier data.

Let $\hat\theta_m$ be the post-update estimate after observation $m$. We
report the tail average:
\[
 \bar\theta_J^{(f)}
 =\frac{1}{J-b_f}\sum_{m=b_f+1}^{J}\hat\theta_m,
 \qquad
 b_f=\lfloor(1-f)J\rfloor,\quad 0<f\leq1.
\]
The choice $f=1$ averages all post-update estimates and corresponds to the
full-history rule analyzed in \LANLASGDReference; $f=0.50$ averages the second
half of the path. Each path processes every observation once, and $f$
determines which post-update estimates enter the reported average.

We fit candidate ASGD paths on Days~16--19 and select $(\gamma,c,f)$ by their fixed-parameter mean
joint log scores on Days~20--22.\footnote{In the simulation study, $\gamma$ is
selected by minimizing Fisher-weighted squared error relative to the known
data-generating parameter. In this application, the true parameter is unknown,
so tuning parameters are selected using joint log scores on a subsequent
validation period.} For each model, we run one candidate ASGD path for
every combination:
\[
 (\gamma,c,f)\in
 \{0.03,0.10,0.30,1,3,10,30,100,300\}
 \times\{0.55,0.60\}
 \times\{1,0.75,0.50,0.25\}.
\]
All candidate paths use the starting values and coordinate scales fixed on
Days~2--8. After selecting the best
combination separately for each model, we rerun that combination over
Days~16--22 and use its chosen tail average to initialize the Day~23 path.
Both models select $c=0.55$ and $\gamma=1$. The NDPP selects $f=0.50$, whereas
the independence model selects $f=0.75$. The reported averages restart on Day~23, but
the step-size counter does not: $m_0=604{,}800$, so the first main-period
update uses step-size index 604,801. The main one-pass fit spans Days~23--51.

On Days~52--58, all estimated coefficients are held fixed. We use the same
mean joint log score, now computed per second over the $604{,}800$ evaluation
seconds. Panel~B reports the computational performance and later-period
predictions of the two fitted models.

\begin{table}[!t]
\centering
\begin{singlespace}
\scriptsize
\setlength{\tabcolsep}{4.5pt}
\renewcommand{\arraystretch}{1.08}
\begin{tabular}{@{}lp{0.70\linewidth}@{}}
\hline
\multicolumn{2}{c}{Panel A: how the data days are used}\\
\hline
Data days & Use\\
\hline
1 & Select 20 high-activity candidate source computers\\
2--8 & Choose the ten-computer panel; set starting values and coordinate scales\\
9--15 & Not used for panel selection, tuning, or main estimation\\
16--19 & Run candidate ASGD settings\\
20--22 & Evaluate candidate estimates with coefficients held fixed\\
16--22 & Rerun the chosen setting to initialize the main fit\\
23--51 & One-pass ASGD fit, with one update per second\\
52--58 & Evaluate both fitted models with coefficients held fixed\\
\hline
\end{tabular}

\vspace{0.7em}

\begin{tabular}{@{}p{0.50\linewidth}rr@{}}
\hline
\multicolumn{3}{c}{Panel B: computation and later-period prediction}\\
\hline
Quantity & Dynamic NDPP & Independence\\
\hline
Number of parameters & 215 & 170\\
Updates during the main fit & 2,505,600 & 2,505,600\\
Time for the main fit (seconds) & 400.65 & 209.91\\
Updates per second & 6,254 & 11,936\\
One-step-ahead log score during the main fit, Days~23--51
  & $-5.70728$ & $-5.73938$\\
Later-period log score, all-main-path average ($f=1$)
  & $-5.61413$ & $-5.61887$\\
Later-period log score, common last-half average ($f=0.50$)
  & $-5.55451$ & $-5.57564$\\
Later-period log score, validation-selected average
  & $-5.55451$ & $-5.61090$\\
Predicted mean number of active sources, validation-selected average
  & 3.22364 & 3.32558\\
\hline
\end{tabular}
\caption{Chronological design, computation, and predictive performance. The
observed mean number of active sources on Days~52--58 is 2.88722. Over the 28
six-hour evaluation blocks, the NDPP has a higher log score than independence
in 18 blocks under full-history averaging and 20 under common last-half
averaging. Main-fit
runtimes were measured on an Apple M1 iMac with 16GB of memory and exclude
raw-data aggregation, calibration, ASGD-setting selection, and later-period
evaluation.}
\label{tab:lanl-asgd-results}
\end{singlespace}
\end{table}

{The main fit processes roughly 2.5 million updates in under seven
minutes, which is feasible at the data's one-second arrival frequency. On
Days~52--58, with the validation-selected averages, the dynamic NDPP assigns
the realized sets a geometric-mean probability 5.8\% higher than independence
does. It also reduces, but does not eliminate, the overprediction of mean
activity.\par}

The one-step-ahead scores use the estimate available before each observation,
whereas the later-period scores hold the corresponding parameter estimates
fixed on Days~52--58.
The NDPP's advantage over independence depends on averaging. With the
full-history average, the overall gain is small and positive on only three of
seven days. Common last-half averaging produces a larger gain and positive
daily differences on six of seven days.

Separately, Figure~\ref{fig:lanl-asgd-running-averages} compares full-history
and last-half averaging within the NDPP along the main estimation path. At
each daily checkpoint, both averages are held fixed and evaluated on the first
six hours of the following day. Last-half averaging improves prediction in
most forecast windows, but the gain changes sign across checkpoints, so the
improvement is not uniform over time.

\begin{figure}[!htbp]
\centering
\includegraphics[width=0.65\linewidth]{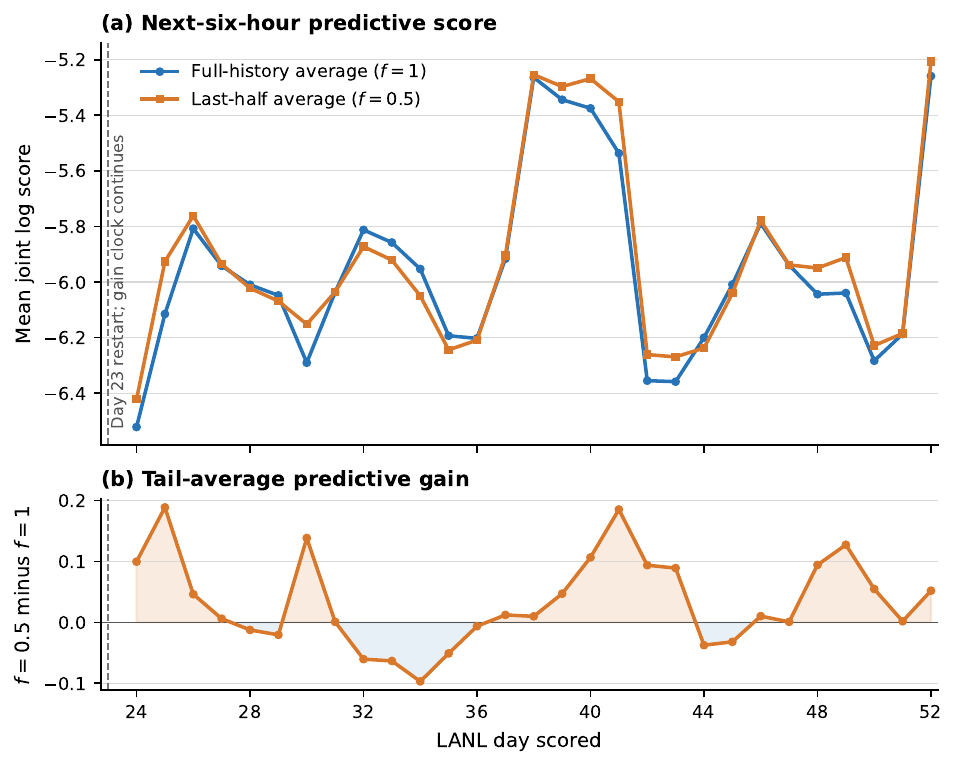}
\caption{Running full-history and last-half ASGD averages for the dynamic
NDPP. The averages restart on Day~23, while the step-size clock includes the
604,800 earlier updates. At the end of each main-fit day, the two averages are
frozen and scored on the first six hours of the next day. Panel~(a) reports
their mean joint log scores; Panel~(b) reports the difference, with positive
values favoring the last-half average. Last-half averaging has a higher log
score in 20 of the 29 forecast windows; its mean and median gains are $0.0341$
and $0.0103$ per second, respectively. The diagnostic does not separate
initialization effects from temporal change.}
\label{fig:lanl-asgd-running-averages}
\end{figure}

\clearpage
\section{Conclusion}
\label{sec:conclusion}

This paper develops a systematic statistical framework for identification,
estimation, and inference in nonsymmetric determinantal point processes. Its
central message is that the ability to
represent positive dependence raises nontrivial identification issues. In {the
class $\mathcal L_\Sigma$ of kernels with a positive semidefinite symmetric part},
observationally equivalent kernels
can form a continuum. With a diagonal symmetric part, only finitely many
likelihood kernels generate a given distribution.

Cardinality truncation identifies the principal minors retained. For statistical
inference, first-order precision depends on both the number of retained observations and their
conditional Fisher information. Under the conditions of Proposition~\ref{prop:asnorm},
truncated maximum likelihood and Polyak--Ruppert ASGD share the same
first-order asymptotic distribution, while the simulations reveal substantial differences
in their finite-sample performance over the reported horizons.
Further research could assess how starting values and learning rates affect
the coverage of ASGD confidence regions.

The applications illustrate this distinction. Full-sample
maximum likelihood remains preferable for the short weekly cyber-risk sample,
whereas one-pass ASGD becomes useful for the multi-million-observation LANL
stream, whose observation-specific kernels make repeated full-data likelihood
evaluation costly. In both settings, the NDPP captures positive
contemporaneous dependence that is absent from the symmetric or independent
DPP models. Nonregular inference at independence and theory for dependent data
streams remain important directions for further research.

\section*{Declaration of generative AI use}
The authors used OpenAI Codex to assist with language editing, manuscript
organization, the review of mathematical arguments, and writing computational
code. The authors reviewed and revised the AI-assisted material and take full
responsibility for the accuracy and integrity of the final manuscript and
accompanying code.

\setlength{\bibsep}{3pt}
\bibliographystyle{apalike_ndpp}
\bibliography{lib}

\clearpage
\appendix
\numberwithin{table}{section}
\section{Proofs}
 
\subsection{Proof of Proposition~\ref{lda}}
\label{app:proofs-est}
{The off-diagonal blocks of $L_0$ and $A_0$ coincide, and
skew-symmetry makes the two cross-block ranks equal. Off-diagonal fullness and
condition~\eqref{simplified} therefore state that $A_0$ is dense and
HL-indecomposable in the sense of \citet{boussairi2015skew}. Let
$\widetilde L=\widetilde\Lambda+\widetilde A\in\mathcal L_\Lambda$ have the
same principal minors as $L_0$ of order at most four. Equal singleton minors
give $\widetilde\Lambda=\Lambda_0$. For a diagonal matrix $\Lambda$ and every
subset $s$:
\[
\det(\Lambda+A)_s=\sum_{u\subseteq s}\det A_u\prod_{i\in s\setminus u}\lambda_i,
\qquad \det A_\varnothing=1,
\]
so, by induction on $|s|$, $\widetilde A$ and $A_0$ have the same principal
minors of order at most four. By \citet[Theorem~7]{boussairi2015skew},
$\widetilde A=DA_0D$ or $\widetilde A=DA_0'D$ for some sign matrix $D$, which
may be normalized to $d_1=1$ because $D$ and $-D$ act identically. Since
$D\Lambda_0D=\Lambda_0$ and $D=D^{-1}$, $\widetilde L=DL_0D^{-1}$ or
$\widetilde L=DL_0'D^{-1}$. Conversely, sign similarity and transposition
preserve membership in $\mathcal L_\Lambda$ and every principal minor, since
$(DL_0D^{-1})_s=D_s(L_0)_sD_s^{-1}$ and $\det((L_0')_s)=\det(L_0)_s$.\par}

It remains to establish the counts. If
$D L_0D^{-1}=E L_0E^{-1}$ for two normalized sign matrices $D$ and $E$, then
$F=E^{-1}D$ commutes with $L_0$. Writing
$F=\operatorname{diag}(f_1,\ldots,f_n)$, off-diagonal fullness gives
$(f_j-f_k)(L_0)_{jk}=0$, so all $f_j$ are equal. Since $f_1=1$, $F=I$ and
$D=E$. Thus each of the two sets of kernels contains $2^{n-1}$ distinct kernels.

The two sets of kernels are disjoint. Otherwise, for normalized sign matrices $D,E$,
we would have $F L_0F^{-1}=L_0'$ with $F=E^{-1}D$. Since $F^{-1}=F$ and
$(L_0)_{jk}\ne0$, comparison of off-diagonal entries gives
$f_jf_k=-1$ for every $j\ne k$. For three distinct indices $i,j,k$, the
product of the three corresponding equalities has left-hand side
$f_i^2f_j^2f_k^2=1$ and right-hand side $-1$, a contradiction.

Finally, in either case, choosing
$d_j=\operatorname{sign}(L_{1j})$ uniquely makes every first-row entry
positive. If $L$ is row-normalized, then
$J L'J$, with $J=\operatorname{diag}(1,-1,\ldots,-1)$, is the row-normalized
representative associated with $L'$. It has the same first row as $L$ and
the negative of every off-diagonal entry in the lower-right block. Since
$L_{23}\ne0$, imposing $L_{23}>0$ selects exactly one of these two
representatives. This proves the proposition.

\subsection{Proof of Proposition~\ref{prop:marginal-parameterization}}
\label{app:proof-marginal-parameterization}
{Fix $R$ and set $\alpha_i=\log\lambda_i$. Then
$p(s)\propto\det(I+R_s)\exp(\alpha'x_s)$, where $x_s\in\{0,1\}^n$ is the
inclusion vector of $s$. Each base weight $\det(I+R_s)$ is positive because
$I+R_s$ has symmetric part $I$. The family is therefore a regular exponential
family with natural parameter space $\mathbb R^n$. It is minimal, since an
affine relation $b'x_s=c$ evaluated at $s=\varnothing$ and at the singletons
gives $c=0$ and $b=0$, and its convex support is $[0,1]^n$. Its mean map
$\alpha\mapsto\E_\alpha X=\operatorname{diag}K$ is therefore a bijection onto
$(0,1)^n$ \citep[Proposition~3.2 and Theorem~3.3]{wainwright2008graphical},
and $\lambda_i=e^{\alpha_i}$ completes the proof.\par}

\subsection{Proof of Proposition~\ref{prop:consistency}}
\label{app:proof-consistency}

\paragraph{Part (i).}
{
Under correct specification, the population criterion is:
\[
  \ell_\infty(\bar n;\theta)
  = \sum_{s:\,|s|\le \bar n}
      \Bigl[\,p(s;\theta_0)\,\log p\bigl(s;\theta\mid S\in \cS_{\bar n}\bigr)\Bigr],
\]
whose maximizing set is:
\begin{align*}
  \bigl\{\theta:\;& p(s;\theta\mid S\in \cS_{\bar n})
         = p(s;\theta_0\mid S\in \cS_{\bar n})
         \text{ for every }s\in\cS_{\bar n}\bigr\}
  = \Theta_0(\bar n),
\end{align*}
by the Kullback--Leibler inequality. For every $\varepsilon>0$, compactness and
continuity give:
\[
  \sup_{\{\theta:\,d(\theta,\Theta_0(\bar n))\geq\varepsilon\}}
  \ell_\infty(\bar n;\theta)
  <
  \sup_{\theta\in\Theta}\ell_\infty(\bar n;\theta),
  \qquad
  \text{where }d(\theta,A)=\inf_{a\in A}\|\theta-a\|.
\]
It remains to verify the required uniform convergence. Write
$p_0(s)=p(s;\theta_0)$ and
$\cS_+=\{s\in\cS_{\bar n}:p_0(s)>0\}$. Almost surely,
$\hat p_T(s)=0$ for every $T$ and every $s\notin\cS_+$. Moreover,
$m=\min_{\theta\in\Theta}\ell_\infty(\bar n;\theta)>-\infty$.
Finiteness of $\ell_\infty$ implies $q(s;\theta)>0$ for every
$s\in\cS_+$ and $\theta\in\Theta$. Since $\log q(s;\theta)\leq0$, we have,
for such $s$ and $\theta$,
\[
  m
  \leq \ell_\infty(\bar n;\theta)
  \leq p_0(s)\log q(s;\theta)
  \leq0.
\]
Thus each $\log q(s;\theta)$ is uniformly bounded on $\Theta$ for
$s\in\cS_+$. Applying the strong law of large numbers to the finitely many
subset indicators gives:
\begin{align*}
  \sup_{\theta\in\Theta}
  \left|T^{-1}\ell_T(\bar n;\theta)-\ell_\infty(\bar n;\theta)\right|
  &\leq
  \sum_{s\in\cS_+}
  |\hat p_T(s)-p_0(s)|
  \sup_{\theta\in\Theta}|\log q(s;\theta)|
  \longrightarrow0
\end{align*}
almost surely. This uniform convergence implies that, for all sufficiently
large $T$, every empirical maximizer lies within $\varepsilon$ of
$\Theta_0(\bar n)$, proving outer consistency.
}

\paragraph{Part (ii).}
{
Since $\theta_0^\star\in\Theta_0(\bar n)$, the conditional laws on
$\cS_{\bar n}$ coincide, and hence
\[
  \nabla_\theta\ell_\infty(\bar n;\theta_0^\star)
  =\pi_{\bar n}(\theta_0)
    \sum_{s\in\cS_{\bar n}}q(s;\theta_0^\star)
      \nabla_\theta\log q(s;\theta_0^\star)
  =\pi_{\bar n}(\theta_0)
    \left.\nabla_\theta\sum_{s\in\cS_{\bar n}}q(s;\theta)
    \right|_{\theta=\theta_0^\star}
  =0.
\]
{\emergencystretch=2em
Here differentiation is performed before evaluation at $\theta_0^\star$;
the derivative vanishes because the conditional probabilities sum to one.\par}

Fix $0<\kappa<\lambda_{\min}(\mathcal H)$. Writing
$e=\theta-\theta_0^\star$ and using
$\nabla_\theta^2
\ell_\infty(\bar n;\theta_0^\star)=-\mathcal H$ gives:
\[
  e'\nabla_\theta
  \ell_\infty(\bar n;\theta_0^\star+e)
  =-e'\mathcal H e+o(\|e\|^2)
  \leq-\kappa\|e\|^2,
\]
on a sufficiently small ball
$B_r\subset\operatorname{int}(\Theta)$ centered at $\theta_0^\star$. Let
$g_t(\theta)=\nabla_\theta\ell(S_t,\bar n;\theta)$, let
$\mathcal F_t$ be the data-and-iterate filtration, and set
$e_t=\hat\theta_t-\theta_0^\star$. The independence of $S_t$ from
$\mathcal F_{t-1}$ gives:
\[
  \E\!\left[g_t(\hat\theta_{t-1})\mid\mathcal F_{t-1}\right]
  =\nabla_\theta
    \ell_\infty(\bar n;\hat\theta_{t-1}).
\]
Positivity and twice continuous differentiability of $q$ over its finite
support give, for some $C<\infty$,
\[
  \E\!\left[\|g_t(\theta)\|^2\mid\mathcal F_{t-1}\right]
  \leq C,
  \qquad \text{for any }\theta\in B_r.
\]
Projection onto the closed convex set $B_r$ is nonexpansive, so the projected
recursion satisfies:
\[
  \E\!\left[\|e_t\|^2\mid\mathcal F_{t-1}\right]
  \leq
  \|e_{t-1}\|^2
  +2\gamma_t e_{t-1}'\nabla_\theta
    \ell_\infty(\bar n;\hat\theta_{t-1})+C\gamma_t^2
  \leq
  (1-2\kappa\gamma_t)\|e_{t-1}\|^2+C\gamma_t^2.
\]
Because $\sum_t\gamma_t=\infty$ and $\sum_t\gamma_t^2<\infty$,
\citet[Theorem~1]{robbins1971convergence} implies that $\|e_t\|^2$ converges
and $\sum_t\gamma_t\|e_{t-1}\|^2<\infty$ almost surely. The almost-sure limit
of $\|e_t\|^2$ must therefore be zero. Ces\`aro
averaging then gives $\bar\theta_T\to\theta_0^\star$. Applying the same
argument to the unprojected recursion stopped upon leaving $B_r$, and then
localizing on the event of eventual confinement, gives the unprojected result.
}

\subsection{Proof of Proposition~\ref{prop:asnorm}}
\label{app:proof-asnorm}

Use the conditional probabilities $q(s;\theta)$ and retained sample size $N_T$
defined in Section~\ref{sec:conditional-truncated-likelihood}.
The one-observation criterion is
$\ell(S_t,\bar n;\theta)=\1\{S_t\in\cS_{\bar n}\}\log q(S_t;\theta)$.
Since $\theta_0^\star\in\Theta_0(\bar n)$, the conditional laws on
$\cS_{\bar n}$ coincide:
$q(\cdot;\theta_0^\star)=q(\cdot;\theta_0)$ on $\cS_{\bar n}$. Hence, for any
function $g$ on $\cS_{\bar n}$:
\begin{equation}\label{eq:a6-master}
  \E_{\theta_0}\!\left[\1\{S\in\cS_{\bar n}\}g(S)\right]
  =\pi_{\bar n}(\theta_0)
    \E_{\theta_0^\star}\!\left[g(S)\mid S\in\cS_{\bar n}\right].
\end{equation}

Each $q(s;\theta)$ is positive and smooth near $\theta_0^\star$, and
$\cS_{\bar n}$ is finite. Differentiating
$\sum_{s\in\cS_{\bar n}}q(s;\theta)=1$ once and twice therefore gives, at
$\theta_0^\star$:
\[
  \E_{\theta_0^\star}\!\left[\nabla_\theta\log q(S;\theta_0^\star)
  \mid S\in\cS_{\bar n}\right]=0,
  \qquad
  \E_{\theta_0^\star}\!\left[-\nabla_\theta^2\log q(S;\theta_0^\star)
  \mid S\in\cS_{\bar n}\right]
  =\mathcal I_{\bar n}(\theta_0^\star),
\]
where the conditional Fisher information is:
\[
  \mathcal I_{\bar n}(\theta)
  =\sum_{s\in\cS_{\bar n}}\frac{1}{q(s;\theta)}
     \nabla_\theta q(s;\theta)
     [\nabla_\theta q(s;\theta)]' .
\]
Define
\[
  \mathcal V
  =\E_{\theta_0}\!\left[\nabla_\theta\ell(S_t,\bar n;\theta_0^\star)
    [\nabla_\theta\ell(S_t,\bar n;\theta_0^\star)]'\right]
\]
and apply \eqref{eq:a6-master} to the score, Hessian, and score outer product:
\begin{equation}\label{eq:a6-HV}
  \mathcal H
  =-\E_{\theta_0}\!\left[
      \nabla_\theta^2\ell(S_t,\bar n;\theta_0^\star)\right]
  =\pi_{\bar n}(\theta_0)\mathcal I_{\bar n}(\theta_0^\star)
  =\E_{\theta_0}\!\left[
      \nabla_\theta\ell(S_t,\bar n;\theta_0^\star)
      [\nabla_\theta\ell(S_t,\bar n;\theta_0^\star)]'\right]
  =\mathcal V .
\end{equation}
As a Gram matrix, $\mathcal I_{\bar n}(\theta_0^\star)$ is positive
semidefinite. The nonsingularity assumption in
Proposition~\ref{prop:asnorm} makes it positive definite. Thus
$\mathcal H=\mathcal V$ is positive definite.
Hence
$\nabla_\theta^2\ell_\infty(\bar n;\theta_0^\star)=-\mathcal H$
is negative definite and hence Hurwitz (all its eigenvalues have strictly
negative real parts), providing the local stability required for
stochastic-approximation averaging.

The scores are i.i.d. with mean zero and variance $\mathcal V$, so:
\[
  T^{-1/2}\sum_{t=1}^T
  \nabla_\theta\ell(S_t,\bar n;\theta_0^\star)
  \xrightarrow{d}N(0,\mathcal V),
\]
and the Hessian satisfies the usual local uniform law of large numbers. For the
truncated maximum-likelihood estimator $\tilde\theta_T$ in case~(a), a Taylor
expansion of the score equation
around $\theta_0^\star$ gives:
\[
  \sqrt T(\tilde\theta_T-\theta_0^\star)
  =\mathcal H^{-1}T^{-1/2}\sum_{t=1}^T
    \nabla_\theta\ell(S_t,\bar n;\theta_0^\star)+o_p(1).
\]
We therefore obtain the sandwich formula:
\[
\sqrt T(\tilde\theta_T-\theta_0^\star)
\xrightarrow{d}N(0,\mathcal H^{-1}\mathcal V\mathcal H^{-1}).
\]

{
For the averaged stochastic-gradient estimator, the conditional drift of the
recursion is the gradient of the population criterion,
$\nabla_\theta\ell_\infty(\bar n;\theta)$.
By assumption, $\hat\theta_T\to\theta_0^\star$ almost surely. If projection is
used, choose $\delta>0$ such that the $\delta$-ball around
$\theta_0^\star$ lies in both the projection set and the neighborhood
where each $q(s;\theta)$ is three times continuously differentiable. Finite support and local
smoothness bound the score norm by some $C_g$ on this ball. Eventually
$\|\hat\theta_{t-1}-\theta_0^\star\|<\delta/2$ and
$\gamma_tC_g<\delta/2$, so the unprojected proposal lies in the projection set.
The projection is therefore inactive after a finite random time.
Continuous differentiability through order three gives:
\[
  \nabla_\theta\ell_\infty(\bar n;\theta)
  =-\mathcal H(\theta-\theta_0^\star)
   +O(\|\theta-\theta_0^\star\|^2).
\]
The score innovations form a martingale-difference sequence whose conditional
covariance converges to $\mathcal V$; finite support supplies the local moment
and Lindeberg conditions. Moreover, $c\in(1/2,1)$ implies:
\[
  t\gamma_t\to\infty,
  \qquad
  \frac{\gamma_t-\gamma_{t+1}}{\gamma_t}=o(\gamma_t).
\]
Then, given the assumed convergence of the raw iterates, the local averaging argument
in the proof of
\citet[Theorem~2]{polyak1992acceleration}
yields:
\[
  \sqrt T(\bar\theta_T-\theta_0^\star)
  =\mathcal H^{-1}T^{-1/2}\sum_{t=1}^T
    \nabla_\theta\ell(S_t,\bar n;\theta_0^\star)+o_p(1),
\]
which is the same first-order representation as the truncated MLE.
}\par

Since $\mathcal H=\mathcal V=\pi_{\bar n}(\theta_0)
\mathcal I_{\bar n}(\theta_0^\star)$, the sandwich reduces to:
\[
  \mathcal H^{-1}\mathcal V\mathcal H^{-1}
  =
  \frac{1}{\pi_{\bar n}(\theta_0)}
  \mathcal I_{\bar n}(\theta_0^\star)^{-1}.
\]
Finally, $N_T/T\xrightarrow{p}\pi_{\bar n}(\theta_0)$, and Slutsky's theorem gives,
for either estimator $\tilde\theta_T$ in the proposition:
\[
  \sqrt{N_T}(\tilde\theta_T-\theta_0^\star)
  \xrightarrow{d}
  N\!\left(0,\mathcal I_{\bar n}(\theta_0^\star)^{-1}\right).
\]
\hfill$\square$

\begin{remark}[Information rank and score inference]
\label{rem:a6-identif}\label{rem:a6-score}
{Write $G_{\bar n}(\theta)$ for the Jacobian of the retained
probability vector $(q(s;\theta))_{s\in\cS_{\bar n}}$. Positivity of the
retained probabilities gives
$\mathcal I_{\bar n}(\theta)=G_{\bar n}(\theta)'
\operatorname{diag}\{q(s;\theta)^{-1}\}G_{\bar n}(\theta)$, so
$\operatorname{Null}\mathcal I_{\bar n}(\theta)
=\operatorname{Null}G_{\bar n}(\theta)$. In $\mathcal L_\Lambda$, with
coordinates $\theta=(\lambda_i,a_{ij}:i<j)$, $\mathcal I_{\bar n}(\theta)$ is
nonsingular whenever all $\lambda_i>0$, all $a_{ij}\neq0$, and
$\bar n\geq2$. Indeed, a null direction $h$ leaves $q(\varnothing;\theta)$,
hence the truncated normalizer and every retained principal minor, unchanged
to first order; singletons then give $D_h\lambda_i=0$, and pairs give
$2a_{ij}D_ha_{ij}=0$. At $A_0=0$, by contrast, every skew score vanishes, so
$\mathcal I_{\bar n}(\theta_0)$ has rank $n$ although $\theta_0$ is
point-identified.

Under an interior null $\theta=\theta_0$, the score central limit theorem
above and $N_T/T\xrightarrow{p}\pi_{\bar n}(\theta_0)$ give
$N_T^{-1/2}U_T(\theta_0)\xrightarrow{d}N(0,\mathcal I_{\bar n}(\theta_0))$.
This proves \eqref{eq:prop6-score} and justifies the region
\eqref{eq:prop6-score-region} when the information is nonsingular. When it
has locally constant rank, the same argument with the Moore--Penrose
pseudoinverse gives \eqref{eq:prop6-score-partial}. At rank-changing points
such as $A_0=0$, the pseudoinverse statistic concerns only first-order score
directions, and we make no confidence-region claim.\par}
\end{remark}

\section{Online Appendix: Additional results}
\subsection{Stabilized ASGD robustness exercise}
\label{app:stabilized-asgd}
\label{sec:simulation-asgd-redesign}

The scalar ASGD recursion examined in Section~\ref{sec:simulation-study} can
retain a substantial initialization transient in poorly conditioned designs.
This variant introduces pilot-based matrix preconditioning,
indexes updates by retained observations, and averages the iterates after a
burn-in period. It also uses slower learning-rate decay.

For each design and truncation level, let \(Y_1,\ldots,Y_N\) be exactly \(N\)
draws from the retained conditional law. The recursion uses the same
\(\vartheta=(\alpha_0,\alpha_1,\eta)'\) coordinates as the regular simulation
design, with \(\rho=\exp(\eta)\). The first
\(m_N=\lceil N^{0.60}\rceil\) observations form a pilot and remain included in
\(N\), although they are not reused as recursive updates. The pilot MLE
\(\widehat\vartheta_{m_N}\) initializes the recursion and yields the frozen
preconditioner
\[
  P_N
  =
  \left\{\widehat{\mathcal I}_{m_N}
  +\frac{\lambda_{\max}(\widehat{\mathcal I}_{m_N})}{9}I_3\right\}^{-1},
  \qquad
  \widehat{\mathcal I}_{m_N}
  =\mathcal I_{\bar n}(\widehat\vartheta_{m_N}).
\]
For \(k=m_N+1,\ldots,N\), the subsequent updates are
\begin{equation}
  \vartheta_k
  =
  \Pi_\Theta\!\left[
    \vartheta_{k-1}+a_{k-m_N}P_N
    \partial_\vartheta\log q_{\bar n}(Y_k;\vartheta_{k-1})
  \right],
  \qquad
  a_j=0.2(1+j/50)^{-0.60}.
  \label{eq:simulation-preconditioned-asgd}
\end{equation}
Only iterates after \(b_N=\lceil N^{0.70}\rceil\) enter the average,
\[
  \overline\vartheta_N^{\mathrm{stab}}
  =\frac{1}{N-b_N}\sum_{k=b_N+1}^N\vartheta_k.
\]
The paired truncated MLE uses the same \(N\) retained observations.

For an estimator \(\widetilde\vartheta\), define normalized Fisher risk by
\[
  \mathcal R_N(\widetilde\vartheta)
  =
  \frac{N}{3}\,
  \E\!\left[
    (\widetilde\vartheta-\vartheta_0)'
    \mathcal I_{\bar n}(\vartheta_0)
    (\widetilde\vartheta-\vartheta_0)
  \right],
\]
and define dependence error on the reported scale by
\[
  \operatorname{RMSE}_\rho(\widetilde\vartheta)
  =
  \left\{
    \E\!\left[(\exp(\widetilde\eta)-\rho_0)^2\right]
  \right\}^{1/2}.
\]

Table~\ref{tab:simulation-asgd-redesign} reports performance relative to the
paired MLE over 200 replications. In the baseline and dense designs, the
\(\rho\)-RMSE ratio is \(4.43\)--\(5.48\) at \(N=16{,}000\) and
\(6.86\)--\(8.88\) at \(N=64{,}000\). The normalized Fisher-risk ratios also
increase with \(N\). This pattern is not due to a poor MLE denominator: at
\(N=64{,}000\), normalized MLE Fisher risk ranges from \(0.93\) to \(1.13\)
across these six cells. Replacing the pilot preconditioner by its oracle
counterpart yields \(\rho\)-RMSE ratios of \(6.65\)--\(8.86\), nearly the
same as the feasible results. The evidence is therefore consistent with a
persistent initialization component under the fixed \(\eta\)-coordinate
ridge-and-gain specification, rather than error in estimating the pilot
information matrix.

The weak-dependence design gives \(\rho\)-RMSE ratios of
\(1.34\)--\(1.54\) at \(N=64{,}000\). Its normalized Fisher-risk ratios are
\(1.18\)--\(1.80\), but the corresponding MLE risks are
\(15.9\)--\(28.9\), far from the regular benchmark of one. No stabilized-ASGD
path activates the projection in any reported cell. Thus, in the coordinates
used by the main simulation, the joint stabilization does not eliminate the
gap with MLE.

\begin{table}[htbp]
  \centering
  \small
  \caption{Stabilized ASGD relative to paired MLE in the regular designs}
  \label{tab:simulation-asgd-redesign}
  \begin{tabular}{lcrr}
    \hline
    Design & \(\bar n\) & \(N=16{,}000\) & \(N=64{,}000\) \\
    \hline
    \multicolumn{4}{l}{\textit{Panel A: normalized Fisher-risk ratio}} \\
    Baseline & 4 & \(734.017\ [564.026,920.869]\)
                 & \(3047.295\ [2209.686,3951.232]\) \\
             & 5 & \(1266.541\ [965.432,1601.564]\)
                 & \(3005.759\ [2164.202,3960.102]\) \\
             & 8 & \(1199.212\ [836.908,1592.544]\)
                 & \(3521.345\ [2406.098,4776.091]\) \\
    Dense    & 4 & \(952.716\ [702.251,1251.617]\)
                 & \(3152.242\ [2173.469,4215.633]\) \\
             & 5 & \(1400.476\ [1002.209,1849.842]\)
                 & \(3181.502\ [2030.623,4558.362]\) \\
             & 8 & \(3884.114\ [2920.663,4985.889]\)
                 & \(6463.940\ [3956.320,9448.690]\) \\
    \hline
    \multicolumn{4}{l}{\textit{Panel B: \(\rho\)-RMSE ratio}} \\
    Baseline & 4 & \(4.483\ [3.964,5.098]\)
                 & \(8.880\ [7.618,10.272]\) \\
             & 5 & \(5.097\ [4.475,5.792]\)
                 & \(7.371\ [6.161,8.665]\) \\
             & 8 & \(4.693\ [4.088,5.403]\)
                 & \(7.655\ [6.459,9.060]\) \\
    Dense    & 4 & \(4.692\ [4.133,5.372]\)
                 & \(7.602\ [6.401,8.896]\) \\
             & 5 & \(4.428\ [3.829,5.129]\)
                 & \(6.927\ [5.659,8.198]\) \\
             & 8 & \(5.477\ [4.779,6.221]\)
                 & \(6.864\ [5.545,8.213]\) \\
    \hline
  \end{tabular}

  \vspace{0.4em}
  \parbox{0.94\textwidth}{\footnotesize
  Panel A reports ratios of Monte Carlo mean normalized Fisher risks. Panel B
  reports ratios of root mean squared errors for \(\widehat\rho=\exp(\widehat
  \eta)\). Brackets give paired percentile-bootstrap \(95\%\) intervals.
  There are 200 paired replications per cell.}
\end{table}

\subsection{Observation-specific kernels and streaming estimation}
\label{sec:simulation-contextual-asgd}

We now let the kernel vary with an observed continuous covariate.
Let \(Z_t\sim U[-1,1]\) independently across \(t\), and,
conditional on \(Z_t\), draw \(S_t\) from the DPP with:
\begin{equation}
  L_t(\theta)
  =D_t(\alpha)(I_8+e^\eta R_0)D_t(\alpha),
  \qquad
  [D_t(\alpha)]_{ii}
  =\exp\left\{\frac{\alpha_0+\alpha_1Z_tx_i}{2}\right\}.
  \label{eq:simulation-contextual-kernel}
\end{equation}
Here \(x_i\) and \(R_0\) are those of
\eqref{eq:simulation-kernel}. We set \(\alpha_{1,0}=0.8\) and
\(\rho_0=e^{\eta_0}=1.5\), and choose \(\alpha_{0,0}=-1.0440\) so that
\(\E|S_t|=2.5\). The matched static design sets \(Z_t=1\) and recalibrates
\(\alpha_{0,0}=-1.1320\) to preserve the same expected cardinality. Thus the
two designs have the same parameterization, dependence strength, gain
schedule, and average set size. In the contextual design, the pairs
\((Z_t,S_t)\) are i.i.d.; conditional on the realized contexts, the subset
laws are independent but non-identical. For
\(d\in\{\mathrm{stat},\mathrm{ctx}\}\), let
\(\overline{\mathcal I}_d
=\E_d[\mathcal I(Z;\theta_{0,d})]\) denote the integrated Fisher information.
Here \(\mathcal I(Z;\theta_{0,d})\) is the Fisher information conditional on
\(Z\), and the static distribution is degenerate at \(Z=1\).

For a common kernel, the log-likelihood can be written as:
\[
  \ell_T(\theta)
  =\sum_{s\subseteq\{1,\ldots,8\}}N_T(s)\log p(s;\theta),
\]
so one counting pass reduces the data to 256 frequencies. With continuous
\(Z_t\), observations having the same realized subset generally have
different normalizers, and the corresponding criterion is instead:
\[
  \ell_T^{\mathrm{ctx}}(\theta)
  =\sum_{t=1}^T
  \left\{\log\det L_{t,S_t}(\theta)
  -\log\det[I_8+L_t(\theta)]\right\}.
\]
{\emergencystretch=2em
Therefore, each full-likelihood evaluation must revisit the covariates and
observation-specific normalizers, whereas post-pilot ASGD evaluates one
component score and then discards the observation.\par}

The reported sample size \(T\) includes an initialization block of
\(m=4{,}000\) observations. An MLE obtained by optimizing from three starting
values and retaining the solution with the highest likelihood initializes both
estimators. Writing \(k=t-m\), ASGD processes observations
\(t=m+1,\ldots,T\) with learning rate \(0.60k^{-0.60}\) and averages all
iterates dated \(t=m+1,\ldots,T\). The paired MLE uses all \(T\) observations.
Both estimators use the full likelihood, and
Table~\ref{tab:simulation-contextual-asgd} reports 200 paired replications per
cell. Because the pairs \((Z_t,S_t)\) are i.i.d. and the distribution of \(Z_t\) does not
depend on the parameter \(\theta\), the argument underlying
Proposition~\ref{prop:asnorm} carries over with the conditional score and the
integrated information \(\overline{\mathcal I}_d\).

To compare joint accuracy across designs, define the Fisher-weighted risk:
\[
  R_{T,d}(\tilde\theta)
  =T\,\E_d\left[
    (\tilde\theta-\theta_{0,d})'
    \overline{\mathcal I}_d
    (\tilde\theta-\theta_{0,d})
  \right].
\]
The factor \(T\) places this risk on the root-\(T\) asymptotic scale.

\begin{table}[!htbp]
  \centering
  \small
  \begin{tabular}{lrrrrr}
    \hline
    \multicolumn{6}{l}{Panel A: estimator accuracy} \\
    Design & \(T\) & MLE RMSE & ASGD RMSE
      & RMSE ratio & Risk ratio \\
    \hline
    Static &    64,000 & 0.02369 & 0.07828 & 3.30 & 4.12 \\
           &   256,000 & 0.01255 & 0.02569 & 2.05 & 2.02 \\
           & 1,000,000 & 0.00706 & 0.01006 & 1.43 & 1.39 \\
    \hline
    Contextual &    64,000 & 0.02450 & 0.07283 & 2.97 & 4.08 \\
               &   256,000 & 0.01216 & 0.02544 & 2.09 & 2.14 \\
               & 1,000,000 & 0.00678 & 0.00938 & 1.38 & 1.34 \\
    \hline
    \multicolumn{6}{l}{Panel B: pilot-inclusive computation at \(T=1{,}000{,}000\)} \\
    Design & Estimator & Obs. visits
      & Enumerated weights & Seconds & MLE evals. \\
    \hline
    Static     & MLE  & 1.004 &   0.017 & 0.001 & 14.57 \\
               & ASGD & 1.000 & 254.990 & 2.430 & -- \\
    Contextual & MLE  & 18.700 & 4787.133 & 24.233 & 18.46 \\
               & ASGD & 1.241 & 317.629 & 2.643 & -- \\
    \hline
  \end{tabular}
  \caption{Accuracy and computation with static and observation-specific
  kernels. Panel A reports RMSE for \(\widehat\rho\); both ratios are ASGD
  divided by MLE within the same design. At \(T=1{,}000{,}000\), paired-bootstrap
  \(95\%\) intervals for the contextual ASGD-to-MLE \(\rho\)-RMSE and
  Fisher-weighted risk ratios are \([1.25,1.53]\) and \([1.22,1.49]\),
  respectively. All MLE optimizations satisfy the
  numerical convergence criterion, every ASGD path completes, and no ASGD path
  activates the projection. In Panel B, observation visits, state-weight
  evaluations, and elapsed seconds include both initialization and main
  estimation. MLE evaluations count only main-stage criterion evaluations.
  Observation visits and state-weight evaluations are in millions. One
  state-weight evaluation is one of the 256 terms used to form a normalizer;
  static evaluations use compressed counts. Elapsed seconds are median
  estimation times from ten single-threaded
  replications on an Apple M1 iMac; data generation is excluded.}
  \label{tab:simulation-contextual-asgd}
\end{table}

Because finite-sample ASGD performance also depends on the curvature of the
objective function, we report the spectral ratio of
\(\overline{\mathcal I}_d\):
\[
  \kappa(\overline{\mathcal I}_d)
  =\frac{\lambda_{\max}(\overline{\mathcal I}_d)}
  {\lambda_{\min}(\overline{\mathcal I}_d)}.
\]
It is \(52.26\)
in the static design and \(41.97\) in the contextual design, indicating similar
conditioning. These values are much smaller than the \(377\)--\(470\) found in
the regular designs of Table~\ref{tab:simulation-asgd-long}, which also use
different initialization and learning-rate decay. Consequently, only the
static--contextual comparison within the present experiment isolates the effect
of observation-specific covariates.

The ASGD accuracy gap contracts substantially with sample size and is similar
in the static and contextual designs. Nevertheless, ASGD remains less accurate
than MLE at the largest horizon. The paired-bootstrap intervals for both
contextual accuracy ratios lie above one, indicating that the remaining gaps
exceed Monte Carlo uncertainty.

The computational ordering reverses across designs. Count compression favors
full-sample MLE for static kernels, whereas observation-specific kernels make
streaming estimation advantageous. In this implementation, ASGD is about nine
times faster than contextual MLE at the largest horizon, at a finite-sample
accuracy cost.

\subsection{Relation to the literature on matrix-valued binary variables}
\label{app:cyber-matrix-literature}
Two related model classes provide context for the specifications in
Section~\ref{sec:cyber-model-specification}.

First, the quadratic-exponential model of \citet{lovison2006matrix} allows
interactions within and between units, but its normalizing constant sums over
all $2^{qp}$ binary configurations. With $q=5$ and $p=4$, this sum has
$2^{20}=1{,}048{,}576$ terms. The NDPP instead uses the determinant
$\det(I_{20}+L_t)$, avoiding this enumeration.

Second, \citet{mariet2016kronecker} impose the symmetric product
$L=L_S\otimes L_B$, with both positive definite factors. To compare with the
NDPP specification, we estimate a symmetric DPP with the
same cell-specific spline intensities, written using the quality--diversity
decomposition:
\begin{equation}
\begin{aligned}
 L_t^{\mathrm{sym}}
 &=D_t^{\mathrm{sym}}
 (C_G^{\mathrm{sym}}\otimes C_B^{\mathrm{sym}})
 D_t^{\mathrm{sym}},\\
 C_G^{\mathrm{sym}}&=(C_G^{\mathrm{sym}})'\succ0,
 & C_B^{\mathrm{sym}}&=(C_B^{\mathrm{sym}})'\succ0,\\
 \operatorname{diag}(C_G^{\mathrm{sym}})
 &=\operatorname{diag}(C_B^{\mathrm{sym}})=\mathbf 1,
\end{aligned}
 \label{eq:cyber-symmetric-benchmark}
\end{equation}
where
$D_t^{\mathrm{sym}}=\operatorname{diag}
\{\exp(z_t'\beta_i^{\mathrm{sym}}/2):i=1,\ldots,20\}$.
This specification is more general than that of
\citet{mariet2016kronecker} because the cell-specific $D_t^{\mathrm{sym}}$ need
not itself factor as a Kronecker product.
The factors contain ten and
six free dependence parameters, respectively, so this symmetric benchmark
has the same 16 dependence parameters as the NDPP axis model in
\eqref{benchmark}. But its symmetric likelihood kernel implies nonpositive
conditional pairwise covariances.

\subsection{Normalized dependence estimates for the PRC application}
\label{app:cyber-selfscaled-parameters}

Table~\ref{tab:cyber-selfscaled-parameters} reports the normalized dependence
estimates discussed in Section~\ref{sec:cyber-eling-jung}.

\begin{table}[!htbp]
\centering
\begin{singlespace}
\scriptsize
\setlength{\arraycolsep}{3pt}
\[
\widehat S=
\begin{array}{c|rrrrr}
 & \mathrm{BSF}&\mathrm{BUS}&\mathrm{MED}&\mathrm{EDU}&\mathrm{PUBLIC}\\\hline
\mathrm{BSF}    & 0     & 1.433 & 1.168 & 1.073 & 0.924\\
\mathrm{BUS}    &-1.433 & 0     & 1.978 &-1.603 &-1.537\\
\mathrm{MED}    &-1.168 &-1.978 & 0     & 1.344 &-0.631\\
\mathrm{EDU}    &-1.073 & 1.603 &-1.344 & 0     & 2.068\\
\mathrm{PUBLIC} &-0.924 & 1.537 & 0.631 &-2.068 & 0
\end{array}
\]
\[
\widehat B=
\begin{array}{c|rrrr}
 & \mathrm{HACK}&\mathrm{ELET}&\mathrm{DISC}&\mathrm{INSD}\\\hline
\mathrm{HACK} & 0     & 1.119 & 1.207 & 1.358\\
\mathrm{ELET} &-1.119 & 0     &-1.218 & 0.677\\
\mathrm{DISC} &-1.207 & 1.218 & 0     &-1.171\\
\mathrm{INSD} &-1.358 &-0.677 & 1.171 & 0
\end{array}
\]
\[
\widehat C_S=
\begin{array}{c|rrrrr}
 & \mathrm{BSF}&\mathrm{BUS}&\mathrm{MED}&\mathrm{EDU}&\mathrm{PUBLIC}\\\hline
\mathrm{BSF}    & 1     &-0.353 &-0.309 &-1.410 &-0.717\\
\mathrm{BUS}    &-0.353 & 1     &-0.395 & 1.174 &-0.847\\
\mathrm{MED}    &-0.309 &-0.395 & 1     &-0.652 &-0.443\\
\mathrm{EDU}    &-1.410 & 1.174 &-0.652 & 1     & 0.943\\
\mathrm{PUBLIC} &-0.717 &-0.847 &-0.443 & 0.943 & 1
\end{array}
\]
\end{singlespace}
\caption{Normalized point estimates for the selected time-adjusted breach-led
NDPP. Rows and columns of $\widehat S$ and $\widehat C_S$ follow
BSF, BUS, MED, EDU, PUBLIC; those of $\widehat B$ follow HACK, ELET, DISC,
INSD. Together, the signed matrices constitute one representative of the
diagonal-sign and transposition equivalence class.}
\label{tab:cyber-selfscaled-parameters}
\end{table}

\subsection{Starting values and coordinate scaling for LANL ASGD}
\label{app:lanl-asgd-implementation}

This appendix gives the initialization and scaling rules used in
Section~\ref{sec:lanl-asgd-application}.
For the NDPP, write
$\theta=(\beta_1',\ldots,\beta_{10}',r')'$, where
$r=(R_{ij}:1\leq i<j\leq10)'$ collects the upper-triangular entries of $R$;
for the independence model, $\theta=(\beta_1',\ldots,\beta_{10}')'$.
Before tuning, we estimate a fixed positive diagonal scaling matrix
$C_{\mathrm{sc}}$ from score magnitudes on Days~2--8 and write:
\[
\theta=\theta_{\mathrm{init}}+C_{\mathrm{sc}}v,
\]
where $\theta_{\mathrm{init}}$ is the starting value.
{Its activity intercepts are smoothed marginal logits on
Days~2--8, its other activity coefficients are zero, and each skew coordinate
$R_{ij}$, $i<j$, starts at a small positive value that increases with the
pair odds ratio of computers $i$ and $j$ on Days~2--8.\footnote{Specifically,
$R_{ij}^{\mathrm{init}}=\min\{0.35,\,0.02+0.10\max(\widehat{\mathrm{OR}}_{ij}-1,0)^{1/2}\}$,
where $\widehat{\mathrm{OR}}_{ij}$ is the odds ratio of the activity
indicators of computers $i$ and $j$ on Days~2--8, with $0.5$ added to each
cell of the $2\times2$ table. The resulting starting values range from
$0.021$ to $0.133$.}
As in the PRC application, the skew start must be nonzero because an ASGD
recursion started at $R=0$ never leaves independence.}
The scalar
ASGD recursion is applied to $v$, so updates in the original coordinates
multiply the score by $C_{\mathrm{sc}}^2$.\footnote{The diagonal entries of
$C_{\mathrm{sc}}^2$ are inverse regularized root mean squared calibration
scores, normalized to have median one and restricted to $[0.1,10]$. These
scales are held fixed throughout tuning and estimation.}

\end{document}